\documentclass[lettersize,journal]{IEEEtran}
\usepackage{amsmath,amsfonts}
\usepackage{algorithmic}
\usepackage{algorithm}
\usepackage{array}
\usepackage[caption=false,font=normalsize,labelfont=sf,textfont=sf]{subfig}
\usepackage{textcomp}
\usepackage{stfloats}
\usepackage{url}
\usepackage{verbatim}
\usepackage{graphicx}
\usepackage{cite}
\usepackage{amssymb}
\usepackage{makecell}
\usepackage{multirow,booktabs}

\newtheorem{theorem}{Theorem}
\newtheorem{lemma}{Lemma}
\newtheorem{corollary}{Corollary}
\newtheorem{proof}{Proof}
\usepackage{graphicx} 

\begin{document}

\title{Communication-Efficient Federated Learning by Exploiting Spatio-Temporal Correlations of Gradients}

\author{Shenlong~Zheng, Zhen~Zhang, Yuhui~Deng,~\IEEEmembership{Member,~IEEE,} Geyong~Min,~\IEEEmembership{Member,~IEEE,}, and Lin~Cui,~\IEEEmembership{Member,~IEEE}
	\thanks{ S. Zheng, Z. Zhang (corresponding author), Y. Deng (corresponding author), and L. Cui are with the College of Information Science and Technology, Jinan University, Guangzhou, China, 510632. E-mail: zsl503503@stu.jnu.edu.cn, zzhang@jnu.edu.cn, tyhdeng@jnu.edu.cn, tcuilin@jnu.edu.cn.}
		
		\thanks{
		G. Min is with the College of Engineering, Mathematics and Physical
		Sciences, University of Exeter, Exeter EX4 4QF, United Kingdom. E-mail:
		g.min@exeter.ac.uk.}}
\markboth{IEEE Transactions on Computers}%
{Shell \MakeLowercase{\textit{et al.}}: A Sample Article Using IEEEtran.cls for IEEE Journals}

\maketitle

\begin{abstract}
	Communication overhead is a critical challenge in federated learning, particularly in bandwidth-constrained networks. Although many methods have been proposed to reduce communication overhead, most focus solely on compressing individual gradients, overlooking the temporal correlations among them. Prior studies have shown that gradients exhibit spatial correlations, typically reflected in low-rank structures. Through empirical analysis, we further observe a strong temporal correlation between client gradients across adjacent rounds. Based on these observations, we propose GradESTC, a compression technique that exploits both spatial and temporal gradient correlations. GradESTC exploits spatial correlations to decompose each full gradient into a compact set of basis vectors and corresponding combination coefficients. By exploiting temporal correlations, only a small portion of the basis vectors need to be dynamically updated in each round. GradESTC significantly reduces communication overhead by transmitting lightweight combination coefficients and a limited number of updated basis vectors instead of the full gradients. Extensive experiments show that, upon reaching a target accuracy level near convergence, GradESTC reduces uplink communication by an average of 39.79\% compared to the strongest baseline, while maintaining comparable convergence speed and final accuracy to uncompressed FedAvg. By effectively leveraging spatio-temporal gradient structures, GradESTC offers a practical and scalable solution for communication-efficient federated learning.
\end{abstract}


\section{Introduction} \label{sec:introduction}

\IEEEPARstart{F}{ederated} Learning (FL) enables multiple clients to collaboratively train a global model without sharing local data, ensuring privacy and reducing risks. In FL, participants train models locally on private data and upload gradients to a central server, which aggregates them to update the global model~\cite{konecny2015federated,mcmahan2017communication}. 
FL has demonstrated advantages in various domains such as mobile technology~\cite{bonawitz2019towards} and edge computing~\cite{abreha2022federated}. However, despite these benefits, FL still suffers from significant communication overhead during iterative training. Frequent model updates can lead to network congestion, especially in resource-constrained environments~\cite{li2020federated}. Managing communication efficiency becomes an even more critical challenge as FL scales to larger networks.


To address the challenge of communication overhead in FL, most existing approaches focus on the uplink phase, where the client transmits gradients to the server~\cite{konevcny2016federated}, since uplink communication is typically more bandwidth-constrained than downlink. Common techniques such as quantization and sparsification reduce uplink communication by lowering the numerical precision of gradients or transmitting only a subset of them. While effective in reducing data size, these methods often overlook the inherent correlations among gradient values~\cite{shanmugarasa2023systematic}. In contrast, correlation-based compression methods aim to preserve important structural relationships within the gradient data. By capturing and exploiting these correlations, such approaches have the potential to achieve high compression ratios and can be integrated with sparsity or quantization to further enhance communication efficiency.

Prior studies have confirmed that gradients exhibit strong spatial correlation, as their effective dimensionality is much lower than their apparent dimensionality~\cite{zhao2023inrank,cosson2023lowrank}. This correlation enables more effective compression of gradients using low-rank approximation or similar techniques~\cite{wang2018atomo,yu2018gradiveq}. Earlier methods primarily utilized the internal correlation within an individual gradient. Recently, some approaches have been proposed that exploit the correlation between gradients across clients for compression~\cite{wang2023svdfed,abrahamyan2021learned}. However, methods that exploit the correlation between gradients across clients are not always reliable in FL, as the correlation between client gradients decreases with the number of local training rounds and varying data distributions. Particularly when the data across clients are non-IID (Independent and Identically Distributed), the gradient directions across clients change significantly, resulting in lower correlation.

\begin{figure*}[!t]
	\centering
	\begin{minipage}{0.32\linewidth}
		\centering
		\includegraphics[width=1.0\linewidth]{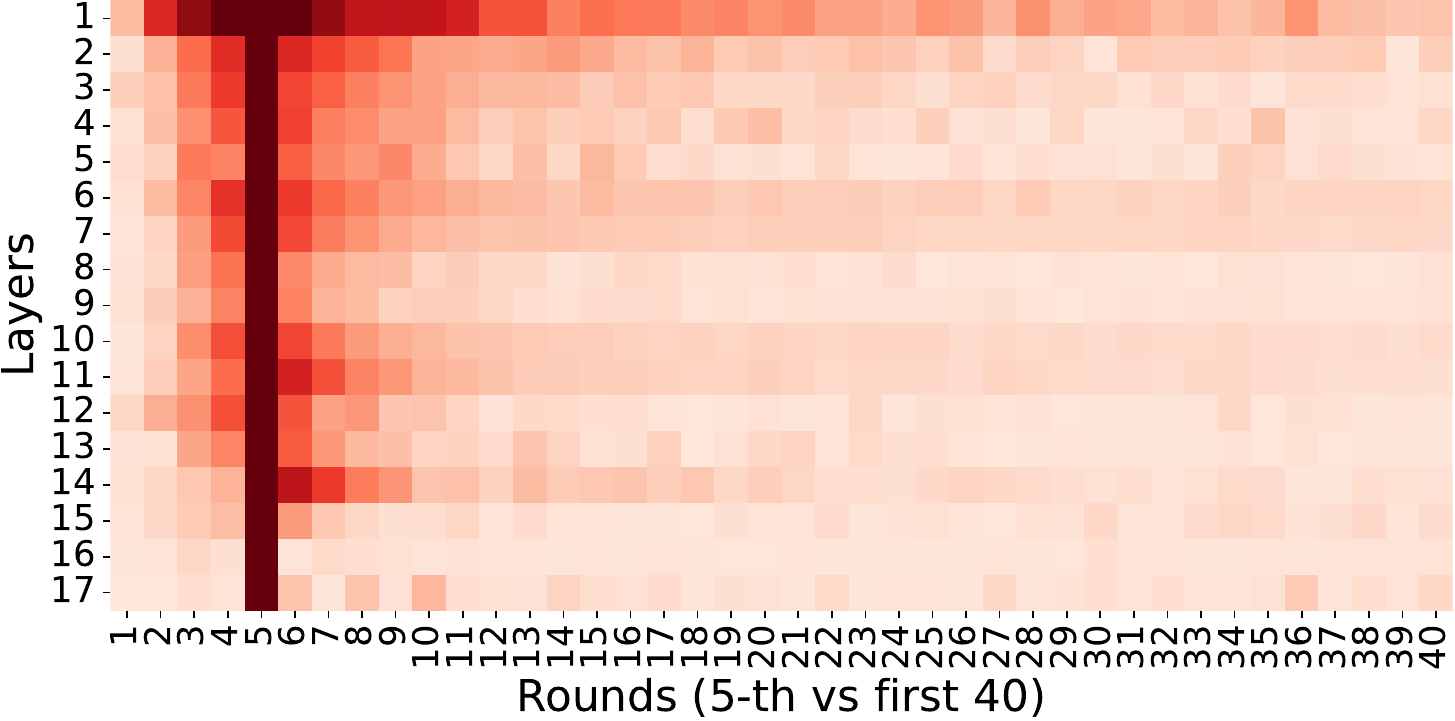}
	\end{minipage}
	\begin{minipage}{0.32\linewidth}
		\centering
		\includegraphics[width=1.0\linewidth]{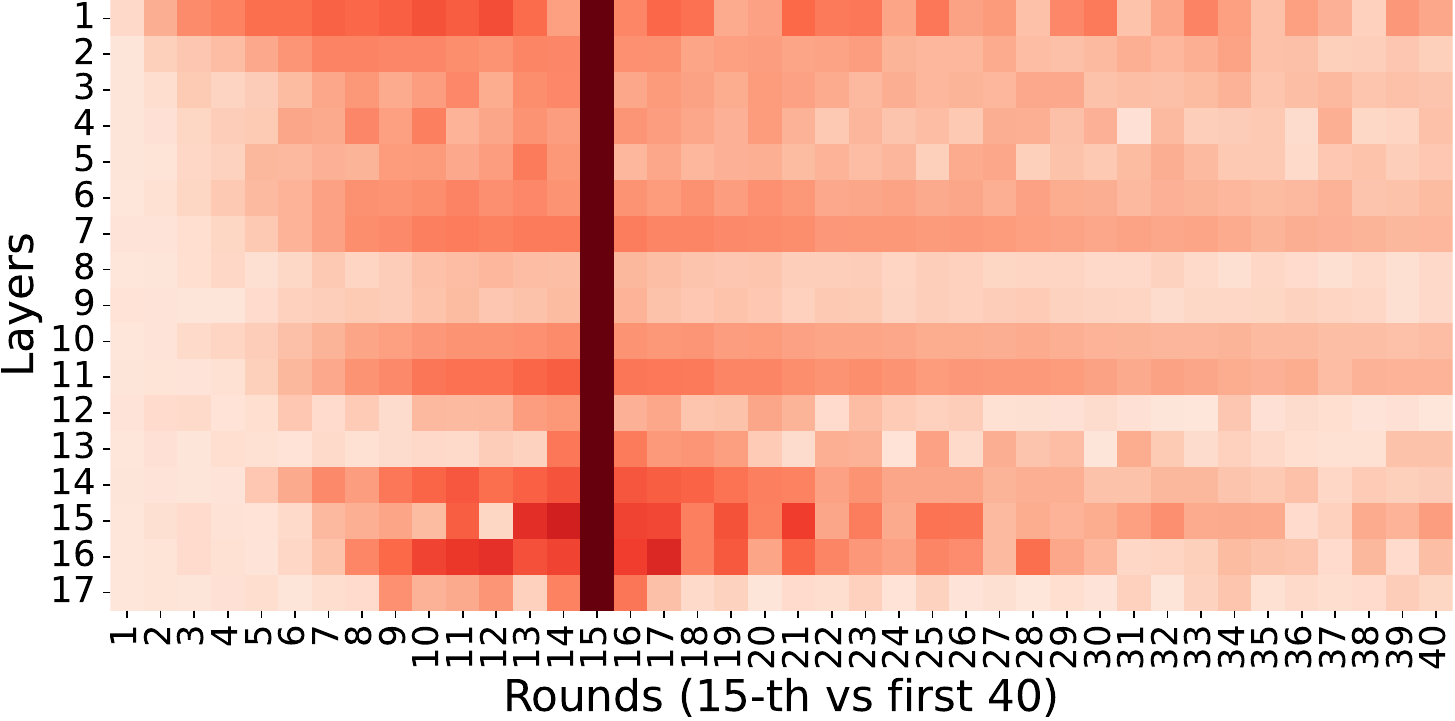}
	\end{minipage}
	\begin{minipage}{0.32\linewidth}
		\centering
		\includegraphics[width=1.0\linewidth]{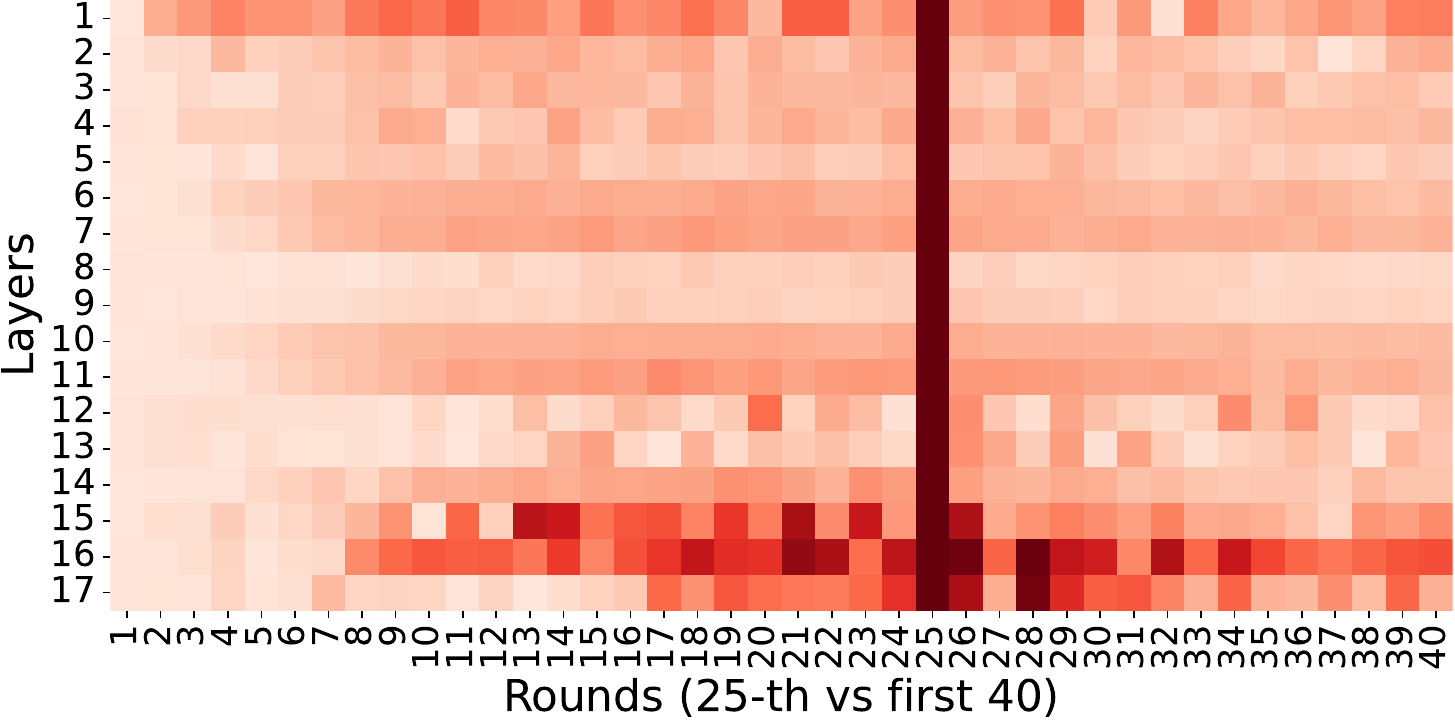}
	\end{minipage}
	
	\begin{minipage}{0.32\linewidth}
		\centering
		\includegraphics[width=1.0\linewidth]{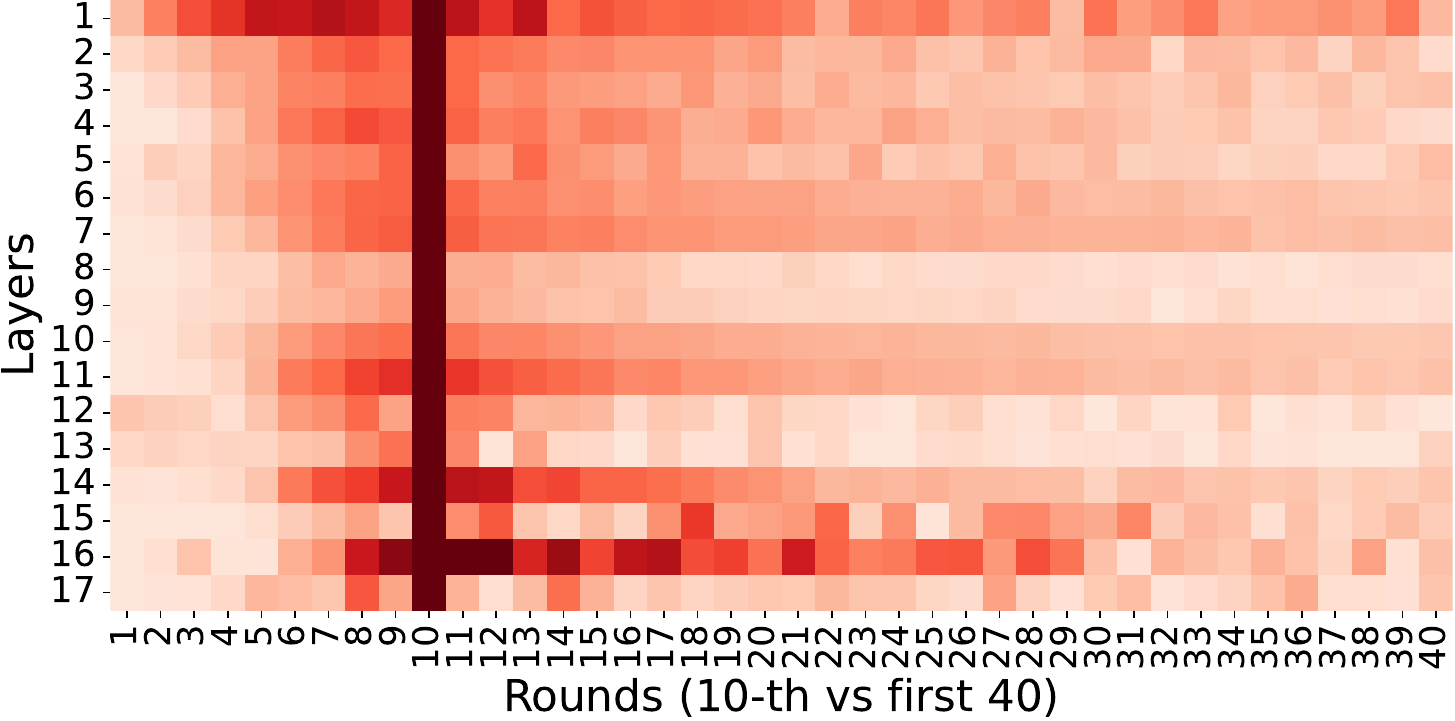}
	\end{minipage}
	\begin{minipage}{0.32\linewidth}
		\centering
		\includegraphics[width=1.0\linewidth]{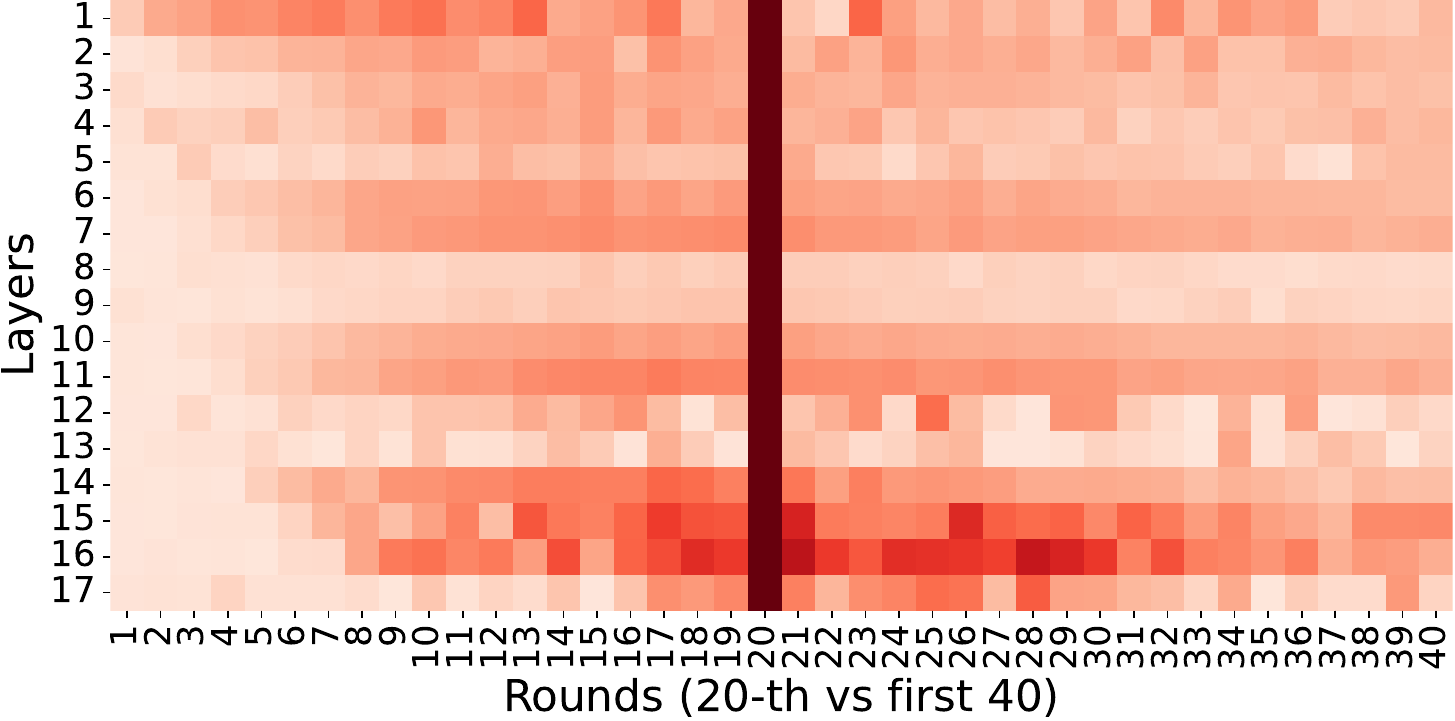}
	\end{minipage}
	\begin{minipage}{0.32\linewidth}
		\centering
		\includegraphics[width=1.0\linewidth]{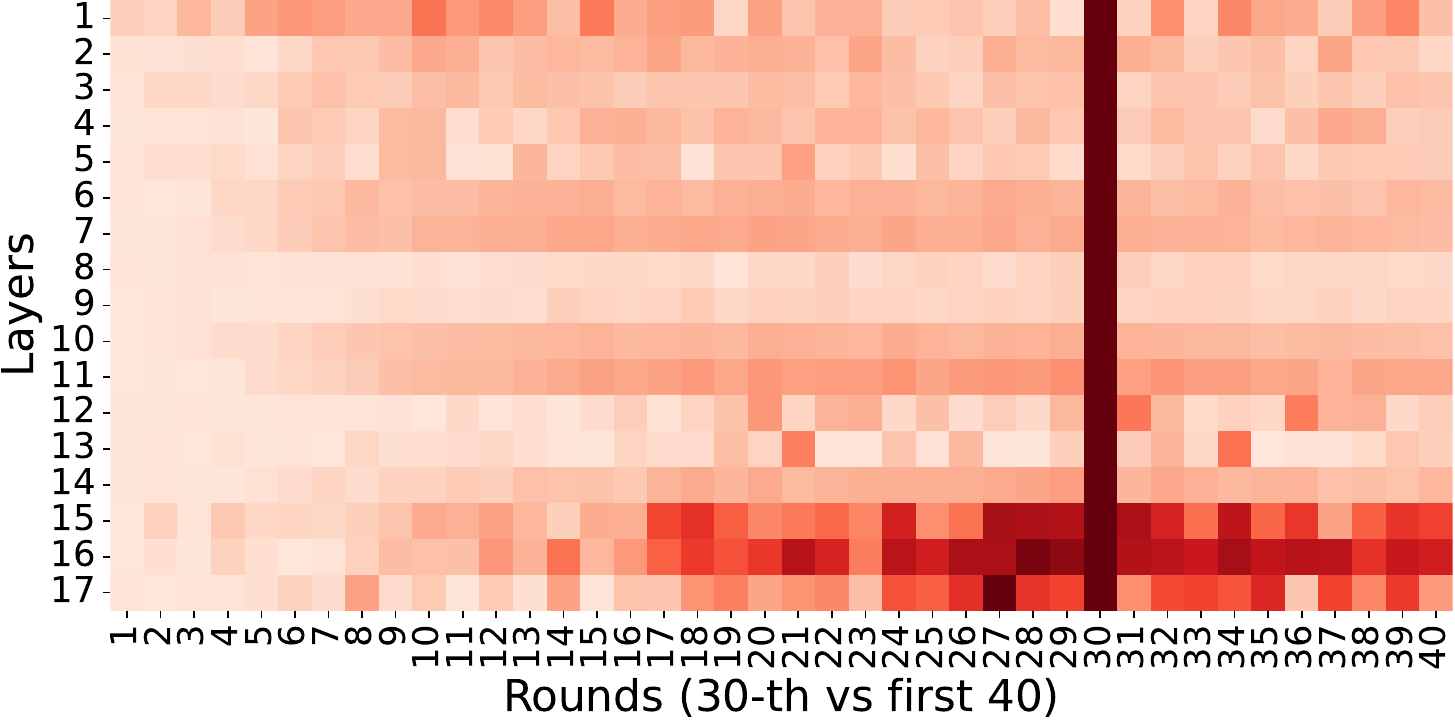}
	\end{minipage}
	\caption{Cosine Similarity Heatmaps of the Client's Gradient Evolution. These heatmaps show the cosine similarity between a client's gradients in the first 40 rounds and specific rounds (5th, 10th, 15th, 20th, 25th, 30th). Darker shades (towards red) indicate higher similarity, and lighter shades (towards white) lower similarity. The deepest red columns indicate self-comparison (cosine similarity = 1). The x-axis represents global rounds, and the y-axis represents model layers, with shallow layers near the raw input.}
	\label{fig:client_similarity}
\end{figure*}
\begin{figure}
	\centering
	\includegraphics[width=0.65\linewidth]{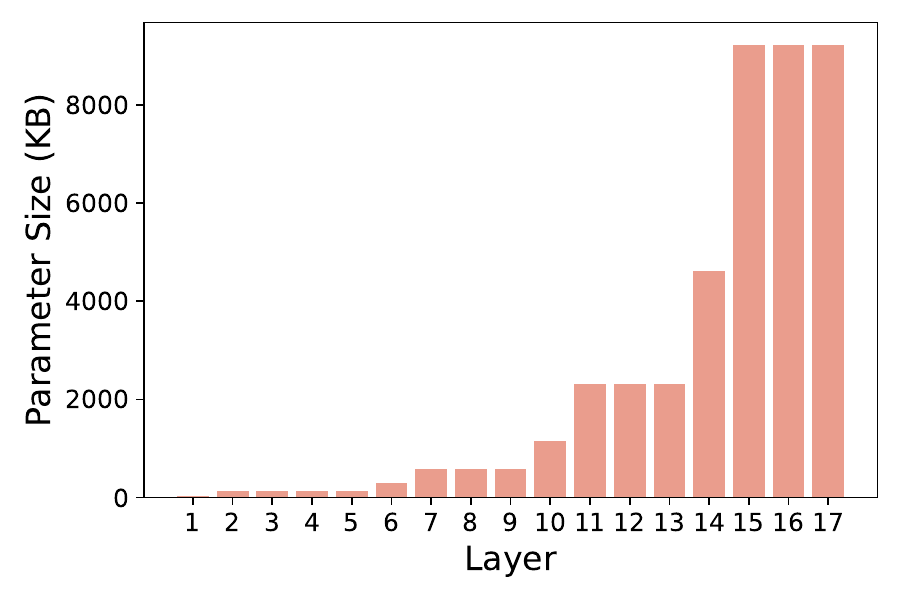}
	\caption{Parameter Size of Each Layer in ResNet18. The x-axis represents the layer index, while the y-axis indicates the parameter size in each layer. }
	\label{fig:layer_params}
\end{figure}


We conduct an empirical analysis to examine whether gradients from the same client across adjacent rounds exhibit exploitable temporal correlation. 
Euclidean distance and mutual information do not directly characterize linear dependence between gradients, whereas subspace-based measures, although capable of doing so, are comparatively expensive and less interpretable.
Therefore, we adopt cosine similarity, a metric that has been widely used in machine learning~\cite{pieterse2019evolving,sattler2020byzantine,jeong2022factorized}, to evaluate the similarity between gradients.
We conducted an experiment with ResNet18 to analyze cosine similarity between gradients from different rounds of 10 clients. Figure~\ref{fig:client_similarity} shows the results for one client randomly selected out of the 10, revealing that gradient similarity evolves over time in a layer-dependent manner. The results reveal that gradients from adjacent rounds exhibit high similarity, indicating a significant temporal correlation. In the early training stage (e.g., round 5 vs. first 40), shallow-layer gradients exhibit relatively higher stability, while deep layers fluctuate more as they have not yet converged. As training proceeds, deep-layer gradients gradually stabilize (e.g., 15-th vs. first 40), maintaining higher similarity between consecutive rounds, whereas shallow-layer gradients become relatively more random due to earlier convergence and smaller updates. This observation is consistent with the hierarchical convergence behavior typically seen in deep learning models~\cite{chen2023layer}.
In addition, as shown in Figure~\ref{fig:client_similarity}, the similarity intensity of different layers exhibits significant differences. The relatively darker regions at the bottom correspond to layers with larger parameter counts, which tend to display stronger temporal correlation. To clarify this relationship, Figure~\ref{fig:layer_params} reports the parameter size of each layer, confirming that these high-similarity regions (e.g., layer 14-17) account for approximately 75\% of the total parameters. These results demonstrate that temporal correlation is not uniform across layers, but rather influenced by both the training stage and layer parameter scale. 
These observations reveals that gradients exhibit strong temporal correlation, which is highly heterogeneous across layers and concentrated in a small subset of parameter-dominant layers that account for the majority of model parameters, making them particularly amenable to compression. Previous works such as Rand-$k$-Temporal~\cite{jhunjhunwala2021leveraging} and FedDCS~\cite{10229032} have exploited the temporal dependency of gradients across rounds. However, Rand-$k$-Temporal implicitly assumes temporal correlation without empirically validating its existence, while FedDCS experimentally confirms the presence of correlation but does not study the factors that affect the strength of the correlation. In contrast, our empirical study provides a layer-wise analysis, demonstrating that temporal correlation is highly heterogeneous and significantly stronger in parameter-dominant layers. This phenomenon is consistently observed across different CNN architectures, including MobileNet and EfficientNet.

Motivated by these findings, we propose GradESTC, which exploits singular value decomposition (SVD) to capture spatial correlations in gradients and adopts a dynamic basis update strategy to exploit temporal correlation across rounds. Moreover, compression is applied primarily to parameter-dominant layers, with more aggressive compression-related hyperparameter settings, thereby reducing computational overhead while improving communication efficiency. To summarize, the main contributions of the work are as follows:

\begin{itemize}
	\item{We conduct an empirical analysis of temporal correlation in FL gradients, showing that gradients from the same client across adjacent rounds exhibit strong temporal correlation, which is highly heterogeneous across layers and is significantly stronger in a small subset of parameter-dominant layers. This observation provides a solid empirical foundation for selectively and layer-wise exploiting temporal correlation in gradient compression.}
	
	\item{We propose GradESTC, a novel approach that exploits spatio-temporal correlations in gradients to achieve efficient communication, naturally suited for non-IID settings. GradESTC maintains a set of basis vectors, representing gradients as combination coefficients of these vectors to capture spatial correlation and dynamically updates these vectors across rounds, capturing temporal correlation with minimal computational overhead.}

	\item{We validate our approach through extensive experiments under both IID and non-IID settings, demonstrating that it achieves communication-optimal performance without compromising learning effectiveness. For example, on CIFAR-10 under the IID setting, it reduces uplink communication by 86.70\% compared to FedAvg while maintaining accuracy on par with the best baselines (see Table~\ref{tab:All_results}).}
	
\end{itemize}

\section{Related Work}
Existing studies have mainly focused on reducing communication overhead through gradient quantization and sparsification, which directly reduce the size of transmitted updates. More recently, researchers have begun to explore the correlations within gradients, providing new opportunities for further compression. In this section, we review work in related areas.

\paragraph{Quantization and Sparsification} Quantization reduces parameter precision to achieve compression. Methods like SignSGD~\cite{bernstein2018signsgd} compress gradients to 1-bit precision. FedPAQ~\cite{pmlr-v108-reisizadeh20a} integrates periodic averaging and quantization, improving communication efficiency while preserving optimality and convergence. Recent approaches, such as QCS-based compression~\cite{li2021communication}, achieve communication-efficient learning by combining quantization and compressive sensing. Sparsification discards less significant gradient components. Methods like Top-k sparsification~\cite{NEURIPS2018_b440509a} transmit only the largest components. STC~\cite{sattler2020robust} combines sparsification, error accumulation, and optimal Golomb coding to compress both uplink and downlink communication efficiently. PruneFL~\cite{jiang2023model} and PD-Adam~\cite{DBLP:journals/tc/YuYZZLCC25} prunes parameters adaptively to reduce overhead.

\paragraph{Correlations in Gradients} Gradient correlation is multifaceted. Early research has pointed out the spatial correlation of gradients. For example, when gradients are vectorized and appropriately stacked into a matrix, they tend to form low-rank structures. Based on this observation, Atomo~\cite{wang2018atomo} was proposed to decompose gradients into sparse atoms, thereby reducing communication overhead. Gradiveq~\cite{yu2018gradiveq} further demonstrated the strong linear correlation among CNN gradients and leveraged principal component analysis to exploit these correlations, significantly reducing gradient dimensionality. GaLore~\cite{zhao2024galore} employed low-rank projection to enable full-parameter learning with reduced memory requirements. Recent studies have investigated gradient correlations across clients for communication-efficient FL. 
LGC~\cite{abrahamyan2021learned} learns a shared gradient representation using a lightweight autoencoder, but it is designed for data-parallel training and is difficult to apply in FL due to limited bandwidth and infrequent communication. 
Inspired by LGC, SVDFed~\cite{wang2023svdfed} captures shared gradient representations across clients via SVD, avoiding autoencoder training and making it more suitable for FL. 
FedOComp~\cite{xue2022fedocomp} exploits gradient correlation through principal component analysis and introduces a two-timescale mechanism that periodically updates a global compression kernel for over-the-air aggregation. 
In addition, temporal correlation of gradients across communication rounds have also been investigated. 
FedDCS~\cite{10229032} further confirms temporal correlation in client gradients and applies it to compressive sensing for communication-efficient FL. 
FedTC~\cite{guan2024fedtc} exploits both spatial and temporal correlations by operating in a fixed transform domain combined with sparsification; however, its reliance on predefined transforms, which does not guarantee optimal energy compression for arbitrary gradient distributions. 



\paragraph{Low-Rank Approximation} Some studies have explored low-rank methods in FL~\cite{nguyen2024towards,hyeonfedpara}, which approximate high-dimensional matrices with lower-dimensional representations. SVD is one of the commonly used techniques, and some recent methods have already applied SVD to reduce communication overhead in FL~\cite{10707306}. However, the full SVD is computationally expensive for large matrices~\cite{kishore2017literature}, making it impractical for the resource-constrained environments typical of FL. Randomized SVD~\cite{halko2011randomized} alleviates this by using random projections to approximate the top singular values and vectors, offering a faster and computationally efficient alternative, making it suitable for FL. 
Several methods, such as SVDFed~\cite{wang2023svdfed} and FedOComp~\cite{xue2022fedocomp}, apply low-rank approximation to obtain globally shared kernels, which is effective for homogeneous data. However, such a global representation may struggle to adapt to non-IID data distributions and evolving client-specific gradient statistics.

Motivated by the above observations and limitations, GradESTC adopts a client-specific, dynamically updated low-rank representation to exploit both spatial and temporal correlations of gradients.

\section{Our Approach}

The overall framework of GradESTC consists of compressors and decompressors. Each client has multiple compressors, corresponding to the selected layers to be compressed in the model. The server has a corresponding decompressor for each compressor, which is used to decompress the information from the specific compressor. Next, we introduce the methods for compression and basis vector maintenance within a compressor-decompressor pair. The corresponding pseudocode for the compressor and decompressor algorithms is given in Algorithms~\ref{alg:algo1} and~\ref{alg:algo2}, respectively.

\begin{table}[b]
	\centering
	\caption{Summary of Key Notations}
	\label{tab:notation}
	\begin{tabular}{ll}
	\toprule
	\textbf{Symbol} & \textbf{Description} \\
	\midrule
	$d$ & Number of candidate basis vectors from fitting error \\
	$k$ & Number of retained basis vectors \\
	$l$ & Row dimension of the reshaped gradient matrix \\
	$m$ & Column dimension of the reshaped gradient matrix \\
	$g$ & The original gradient vector \\
	$\alpha, \beta$ & Parameters for dynamically adjusting $d$ \\
	$R$ & Contribution scores of basis vectors \\
	$G \in \mathbb{R}^{l \times m}$ & The reshaped gradient matrix \\
	$M \in \mathbb{R}^{l \times k}$ & The basis matrix composed of $k$ orthogonal basis vectors \\
	$A \in \mathbb{R}^{k \times m}$ & The combination coefficients for the basis matrix $M$ \\
	$E \in \mathbb{R}^{l \times m}$ & The fitting error matrix, $E = G - M A$ \\
	$\mathbb{P}$ & Indices of basis vectors in $M$ to be replaced \\
	$\mathbb{M}$ & New basis vectors selected from fitting error \\
	$\mathbb{A}$ & New combination coefficients selected from fitting error \\
	$\mathbb{C}$ & Total communication overhead per iteration \\
	\bottomrule
	\end{tabular}
\end{table}

\begin{figure}[!ht]
 	\centering
	\includegraphics[width=0.7\linewidth]{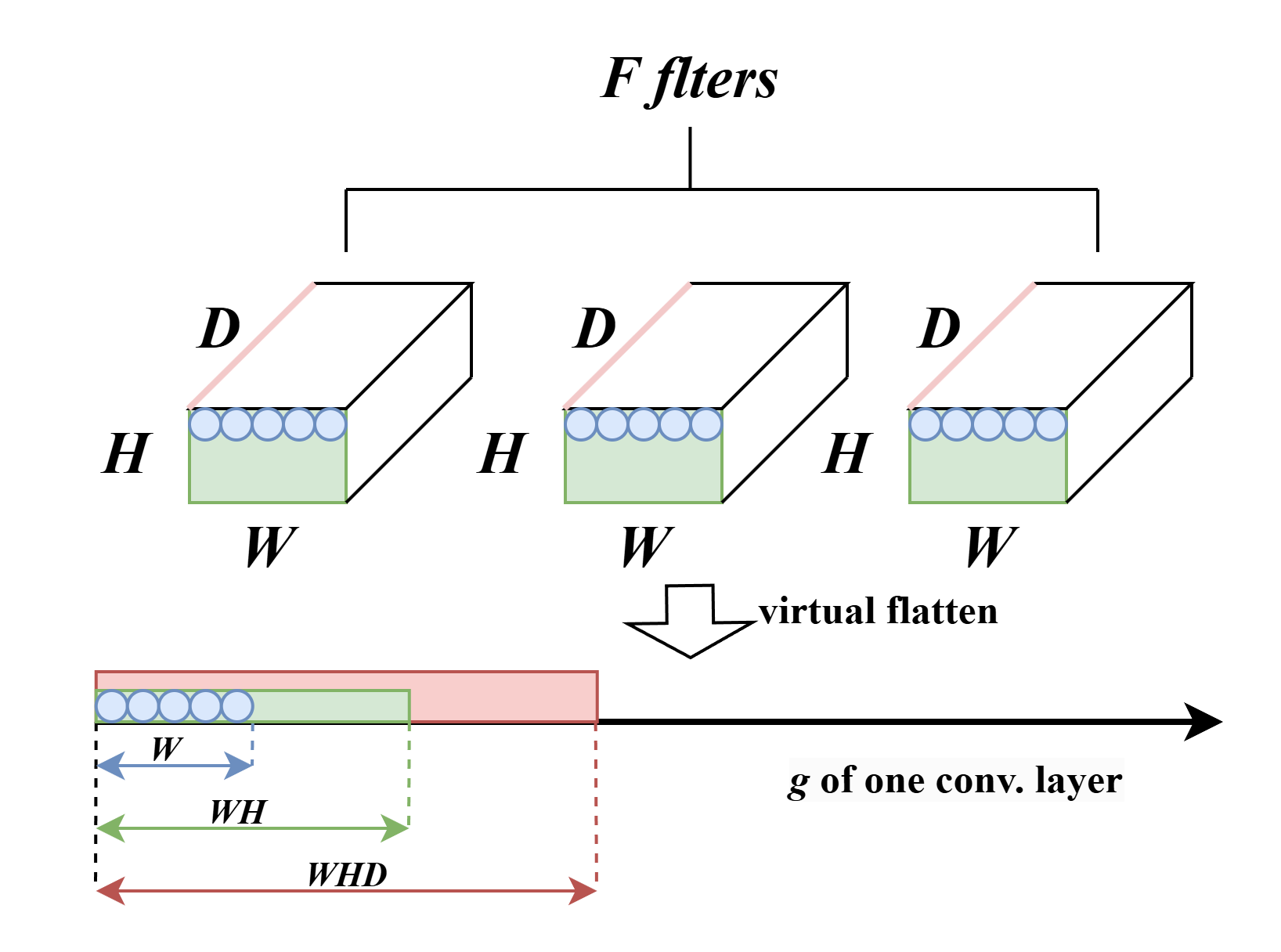}
	\caption{Gradient flattening with WHDC ordering.}
	\label{fig:gradient_preprocessing}
\end{figure}

\subsection{Compression by Short Basis Vectors} \label{sec:compression}
In this section, we assume that a set of orthogonal basis vectors has been obtained to represent the main components of the gradient and introduce the compression process. 

\paragraph{Gradient Preprocessing} 
To explicitly exploit the spatial correlation inherent in gradients, we first reshape each gradient tensor into a matrix that preserves its natural structural boundaries (e.g., channels or kernels). 
Such reshaping transforms the layer-wise gradient tensor into a 2D representation, where spatially or structurally related elements are arranged contiguously. 
This design allows subsequent low-rank approximation to capture correlations among spatially neighboring parameters, rather than treating all entries as independent.

Specifically, gradients are converted into a matrix of a specified shape before compression and are restored after decompression. 
In real-world scenarios, gradients are typically multi-dimensional tensors (e.g., convolutional kernels with four dimensions: Width $W$, Height $H$, Depth $D$, and Channel $C$). 
To facilitate matrix-based compression, we first flatten these tensors into a 1D vector $g \in \mathbb{R}^n$ using the WHDC ordering, as illustrated in Figure~\ref{fig:gradient_preprocessing}. 
Then, the vector $g$ is reshaped into a matrix $G \in \mathbb{R}^{l \times m}$ with $n = l \times m$, where each column $G_{:,j}$ is constructed from consecutive segments of $g$: 
$G_{:,j} = [ g_{(j-1)l + 1}, g_{(j-1)l + 2}, \ldots, g_{jl} ]^\top$ for $j = 1, 2, \ldots, m$.

The reshaping dimension $l$ is carefully chosen so that each column of $G$ aligns with natural spatial or structural boundaries in the network, such as convolutional kernels or feature map channels. 
This alignment preserves the local structure within the gradient and enhances the model's ability to capture spatial correlations~\cite{wang2018atomo}. 
Consequently, the compression process becomes more efficient by representing $G$ as a combination of a few basis vectors that encode the dominant spatial patterns. 
After decompression, the reconstructed matrix $\hat{G}$ is reshaped back to the original tensor form to restore the gradient.

\begin{algorithm}[!t]
    \caption{The compressor on the client side}
    \label{alg:algo1}
    \textbf{Input}: The new gradient $g$ of specified layer\\
    \textbf{Parameter}: $M$, $k$, $d$, $\alpha$, $\beta$\\
    \textbf{Output}: $\mathbb{P}$, $\mathbb{M}$, $A$
    \begin{algorithmic}[1] 
        \STATE $G$ = The segmented matrix from $g$.
        \IF {$M$ not initialized}
            \STATE $U, \Sigma, V$ = Randomized SVD $E$ for first $k$ items.
            \STATE $M$ = The first $k$ columns of $U$.
            \STATE $A$ = The first $k$ rows of $\Sigma V^\top$.
            \STATE $\mathbb{P}$ = The set of the column indices of $M$.
            \STATE $\mathbb{M}$ = The set of the column vectors of $M$.
            \STATE $d = k$ .
        \ELSE
            \STATE $A = M^\top G$.
            \STATE $E = G - M A$.
            \STATE $U^e, \Sigma^e, V^e$ = Randomized SVD $E$ for first $d$ items.
            \STATE $M^e$ = The first $d$ columns of $U^e$.
            \STATE $A^e$ = The first $d$ rows of $\Sigma^e {V^e}^\top$.
			\STATE $M^{oe} = [M, M^e]$.
			\STATE $A^{oe} = [A^\top, {A^e}^\top]^\top$.
            \STATE $R_k$ = Top $k$ elements in $[\| A^{oe}_{1,:} \|^2, \ldots, \| A^{oe}_{k+d,:} \|^2]$.
            \STATE Initialization $\mathbb{P}$, $\mathbb{M}$ and $\mathbb{A}$ as empty sets.
			\FOR{$i=1$ to $k+d$}
				\IF{$i \leq k$ and $R[i] \notin R_k$}
					\STATE $\mathbb{P}$ = $\mathbb{P} \cup \{i\}$.
				\ENDIF
				\IF{$i > k$ and $R[i] \in R_k$}
					\STATE $\mathbb{M} = \mathbb{M} \cup \{M^{oe}_{:,i}\}$.
					\STATE $\mathbb{A} = \mathbb{A} \cup \{A^{oe}_{i,;}\}$.
				\ENDIF
			\ENDFOR
			\STATE Update $M, A$ according to Formula~\ref{eq:update} using $\mathbb{P}, \mathbb{M}, \mathbb{A}$.
			\STATE $d = \text{min}(\alpha \text{sizeof}(\mathbb{M}) + \beta, k)$.
        \ENDIF
        \STATE \textbf{return} $\mathbb{P}, \mathbb{M}, A$.
    \end{algorithmic}
\end{algorithm}

\paragraph{Compression} To compress the gradient matrix $ G $, the basis vectors stacked into the matrix $ M \in \mathbb{R}^{l \times k} $ are used to represent the gradient matrix $ G $, where each column of $M$ represents a basis vector. The goal is to find the combination coefficients $ A \in \mathbb{R}^{k \times m} $ such that the reconstructed gradient $ \hat{G} = M A $. Our objective is to minimize the reconstruction error, expressed as the following optimization problem:
\begin{align}
	\min_A \|G - M A\|^2.
\end{align}
Taking the derivative of the objective function with respect to $A$, we obtain
\begin{align}
	\frac{\partial}{\partial A}\|G - M A\|^2
	&= \frac{\partial}{\partial A}\operatorname{Tr}\left[(G - M A)^\top (G - M A)\right] \nonumber \\
	&= -2M^\top(G - M A).
\end{align}
Setting the derivative to zero yields the first-order optimality condition
\begin{align} \label{eq:me_zero}
M^\top(G - M A) = 0.
\end{align}
Since the columns of $ M $ are orthonormal (i.e., $ M^\top M = I $), the solution for $A$ is:
\begin{align}
	A = M^\top G, 
	\label{eq:alpha}
\end{align}
where the optimal combination coefficients $A$ are obtained. In GradESTC, the compressed representation consists of $ M $ and $ A $. The values of $A$ are transmitted to the server, while $ M $ is maintained as described in Section~\ref{sec:basis_vectors}.

\paragraph{Decompression} On the decompressor, the gradients are reconstructed using the basis matrix $ M $ and the combination coefficients $ A $ as $ \hat{G} = M A $.
Finally, $ \hat{G} $ is reshaped back into the shape of the original gradient, which serves as the decompressed gradient. This reconstruction process ensures that the original gradient is approximated with minimal loss.

\subsection{Maintain Basis Vectors} \label{sec:basis_vectors}
We employ SVD as a tool to factorize the reshaped gradient matrix and construct the basis matrix $M$, which provides the optimal low-rank approximation under a mean squared error criterion. However, as training progresses, gradient statistics gradually evolve, rendering a static basis increasingly suboptimal. Recomputing a full SVD and re-transmitting an entirely new basis matrix at every round would therefore incur substantial computational and communication overhead.

Motivated by the empirical analysis in Section~\ref{sec:introduction} (Figures~\ref{fig:client_similarity} and~\ref{fig:layer_params}), which shows that the gradient direction remains relatively stable across adjacent global rounds, providing a theoretical basis for accurately representing gradients using a set of compact basis vectors that evolves slowly over time. Accordingly, GradESTC retains most previously learned basis vectors locally and dynamically replaces only those whose contribution diminishes relative to the current fitting error, thereby preserving the favorable reconstruction properties of SVD while significantly reducing per-round computation and uplink overhead.

\paragraph{Initialization of Basis Vector} 
At the first compression, the basis matrix $M$ has not yet been obtained. To capture the most significant directions of variance in $ G $ , we apply SVD, which factorizes $ G $ into three components:
\begin{align}
G = U \Sigma V^\top,
\end{align}
where $ U \in \mathbb{R}^{l \times m} $ is an orthogonal matrix whose columns are the left singular vectors of $ G $, $ \Sigma \in \mathbb{R}^{m \times m} $ is a diagonal matrix containing the singular values of $ G $ in descending order, and $ V \in \mathbb{R}^{m \times m} $ is an orthogonal matrix whose columns are the right singular vectors of $ G $. We retain only the first $ k $ columns of $ U $ to form the basis matrix $ M $ for compression. These columns correspond to the $ k $ largest singular values, capturing the most significant variance in $ G $. Furthermore, since $ M $ is constructed from the first $ k $ columns of $ U $, the combination coefficients can be obtained from the first $ k $ rows of $ \Sigma V^\top $ , which is equivalent to being obtained by Formula~\ref{eq:alpha}, thereby reducing the computational burden.


\begin{algorithm}[!t]
    \caption{The decompressor on the server side}
    \label{alg:algo2}
    \textbf{Input}:  $\mathbb{P}, \mathbb{M}, A$\\
    \textbf{Parameter}: $M$\\
    \textbf{Output}: Reconstructed gradient $\hat{g}$ 
	
    \begin{algorithmic}[1] 
        \STATE Update $M$ according to Formula~\ref{eq:update} using $\mathbb{P}$ and $\mathbb{M}$.
        \STATE $\hat{G} = M^\top A$.
        \STATE $\hat{g}$ = Reshape $\hat{G}$ to the original shape.
        \STATE \textbf{return} $\hat{g}$
    \end{algorithmic}
\end{algorithm}

\paragraph{Incremental Replacement of Basis Vector} 
Based on the temporal correlation of gradients, previously learned basis vectors may continue to effectively represent new gradient information. To reduce communication overhead, we propose an incremental replacement strategy that selectively updates the basis vectors by replacing outdated subsets identified via their indices.

Two key problems arise: (1) maintaining the orthogonality of the basis matrix when updating the basis vectors, and (2) deciding which vectors to replace. To address the first problem, we introduce the vectors in the error basis matrix $ M^e $ from fitting error $E \in \mathbb{R}^{l \times m} $ as candidate basis vectors, which captures the components of $ G $ not represented by $ M $. The fitting error $ E$ is computed as:
\begin{align} \label{eq:get_error}
E = G - M A,
\end{align}
where \( A \) is obtained from Formula~\ref{eq:alpha}.

Next, SVD is applied to \( E \), yielding \( E = U^e \Sigma^e {V^e}^\top \). To construct the error basis matrix \( M^e \), we select the columns of \( U^e \) corresponding to singular values greater than zero and choose the first \( d \) columns, where \( d \leq k \). These columns form \( M^e \), which serves as the set of candidate vectors for updating the existing basis \( M \).
It can be rigorously shown that \( M^e \) is orthogonal to \( M \). We obtain $M^\top E = 0$ from the Formula~\ref{eq:me_zero} and~\ref{eq:get_error}, so we have
\begin{align} 
M^\top e_j = 0, \; \forall j,
\end{align}
where \( e_j \) denotes the \( j \)-th column of \( E \). By the SVD, each column of \( E \) can be expressed as a linear combination of the first \( d \) left singular vectors:
\begin{align} 
e_j = \sum_{i=1}^{d} \sigma_i v_{ij} u_i^e,
\end{align}
where \( \sigma_i \) are the non-zero singular values and \( v_{ij} \) are elements of \( {V^e}^\top \). Consequently, for any \( x \in \operatorname{Col}(M^e) \), we have \( x = \sum_j \alpha_j e_j \), which implies:
\begin{align} 
M^\top x = \sum_j \alpha_j M^\top e_j = 0 \quad \Rightarrow \quad M^\top M^e = 0.
\end{align}
The combination coefficients \( A^e \) for \( M^e \) can be computed as:
\begin{align}
A^e = {M^e}^\top G = {M^e}^\top (E + M A) = {M^e}^\top E,
\end{align}
which are equivalently obtained from the first \( d \) rows of \( \Sigma^e {V^e}^\top \), representing the contribution of the error basis vectors in reconstructing \( G \).

To address the second problem, we use the squared norm of each element in the combination coefficients as an indicator of the contribution of the corresponding basis vector. Let $ A^{oe} = [A^\top, {A^e}^\top]^\top, M^{oe} = [M, M^e] $ and let $ A_{u,:} $ denote the $ u $-th row of $ A $, which contains the combination coefficients of the $ u $-th basis vector. Thus, the total contribution of $ A^{oe} $ is captured by a vector $R$, where each element represents the squared norm of a column in $ A^{oe} $:
\begin{align}
R = [\| A^{oe}_{1,:} \|^2, \| A^{oe}_{2,:} \|^2, \ldots, \| A^{oe}_{k+d,:} \|^2].
\end{align}
Although this indicator does not precisely capture the true contribution of the basis vectors to the reconstruction, it provides a computationally efficient approximation that proves effective in most practical scenarios. The contributions sorted in descending order and selected the top $ k $ value as a set $R_k$. We then form ordered sets \( \mathbb{M} \) and \( \mathbb{A} \) by selecting columns from \( M^e \) and rows from \( A^e \) whose corresponding contribution values are in \( R_k \). The indices of vectors in $ M $ that fall outside the top $ k $ contribution form a set denoted as $ \mathbb{P} $. For each \( p_i \in \mathbb{P} \), the \( i \)-th elements from \( \mathbb{M} \) and \( \mathbb{A} \) are sequentially used to replace the corresponding columns in \( M \) and rows in \( A \), respectively, as formally defined below:
\begin{align}\label{eq:update}
	\begin{split}
		M^*_{:,j} &= 
		\begin{cases} 
		\mathbb{M}[i], & \text{if } j = p_i \in \mathbb{P}, \\
		M_{:,j}, & \text{otherwise},
		\end{cases} \quad \text{for } j = 1, 2, \ldots, k,\\
		A^*_{j,:} &= 
		\begin{cases} 
		\mathbb{A}[i], & \text{if } j = p_i \in \mathbb{P}, \\ 
		A_{j:}, & \text{otherwise},
		\end{cases}	\quad \text{for } j = 1, 2, \ldots, k.
	\end{split}
\end{align}
The sets $ \mathbb{P} $ and $ \mathbb{M} $, along with the updated coefficients $ A^* $, are transmitted from the compressor to the decompressor. Using this information, the decompressor updates its basis matrix $ M $ and reconstructs the gradient with the updated basis vectors and combination coefficients.

\paragraph{Dynamic Adjustment of Candidate Vector Count} A full SVD involves high computational overhead. Fortunately, since only the first $d$ basis vectors are needed in our algorithm, an iterative SVD method, called Randomized SVD~\cite{halko2011randomized}, can be employed to reduce this overhead. The time complexity of Randomized SVD depends on $d$, allowing further reductions by decreasing $d$ as described in Section~\ref{sec:time_complexity}. Therefore, we can dynamically adjust $d$ to further reduce computational overhead. We aim to make $d$ as close as possible to the actual number of updates required, while leaving sufficient margin to ensure enough candidate vectors are available and should adapt quickly to minimize computational overhead. In practice, it has been observed that the number of basis vectors actually replaced in each round (i.e., the cardinality of the set $\mathbb{M}$, defined as $d_r$) changes smoothly. Under these requirements, we use a simple linear adjustment strategy, where $d$ is updated as follows:
\begin{align}
d^* = \alpha  d_r + \beta.
\end{align}
This simple strategy achieves good results in most cases while rapidly adjusting $d$ in situations with significant changes, without introducing additional computational overhead. Empirically, we set $\alpha = 1.3$ and $\beta = 1$, which yielded good performance in our experiments.

\subsection{Compression Level and Time Complexity} \label{sec:time_complexity} 
\paragraph{Compression Level} In GradESTC, the compression level is governed by the number of basis vectors \( k \), while the reconstruction loss under a fixed \( k \) depends on the segment length \( l \). Both \( k \) and \( l \) serve as tunable hyperparameters that can be adapted for different layers based on practical needs. In practice, the data for updating basis vectors is dynamically adjusted based on the correlation between consecutive rounds, whereas the data for combination coefficients remains constant. Based on this, the transmitted data is computed as $ \mathbb{C} $:
\begin{align}
\mathbb{C} = k \frac{n}{l} + d_r l + k \leq k (\frac{n}{l} + l + 1).
\end{align}
Generally, $l$ is set to approximately the square root of $n$, aligning with natural structural boundaries, while $k$ is empirically determined with $k \ll l$. This setup ensures a significant compression, i.e., $\mathbb{C} \ll n$.

\paragraph{Time Complexity} The time complexity analysis of the GradESTC focuses on the dynamic incremental replacement process. 
In this process, the fitting error $E$ is computed as Formula~\ref{eq:get_error}, which has a time complexity of $O((2k + 1) l m)$. The time complexity of performing Randomized SVD on $E$ to extract the first $d$ components is $O(\log{(d)}lm  + d^2 (l + m))$, as stated in~\cite{halko2011randomized}. Calculating $A^e$ and $R$ has a time complexity of $O(d m+(k + d) m)$ and the remaining operations can be ignored. Thus, the total time complexity of the dynamic incremental replacement process is as follows: 
\begin{align}
	O\left((2k + \log{(d)} + 1 + \frac{d + k}{l}) l m + d^2 (l + m)\right).
\end{align}
Since $d < k$ and $k \ll l$, certain terms can be ignored, and the complexity can be approximated as $O\left(2k l m+d^2 (l + m)\right)$.
This analysis shows that the time complexity of GradESTC primarily depends on the parameters $k$ and $d$, which can be adjusted to manage computational overhead. During dynamic adjustment, $d$ is adapted as needed, reducing computational overhead significantly while preserving high accuracy.

\section{Convergence Analysis} \label{sec:convergence_analysis}

We consider the optimization problem of FL, where the global objective function is defined as $f(x) \triangleq \frac{1}{N} \sum_{i=1}^{N} f_i(x)$, with $f_i(x)$ representing the local objective function of the $i$-th client and $N$ being the total number of clients. The goal is to minimize the global objective function $f(x)$ by updating the global model $x$ based on local gradients. We analyze the convergence of GradESTC under the non-convex FL setting.

\paragraph{Notation and Assumptions} This paper assumes that each client can locally observe unbiased independent stochastic gradients given by $ g_i^t \triangleq g_i(x^t) $, with $ \mathbb{E}[g_i^t] = \nabla f_i(x^t) $, where $t=1,\ldots,T $ represents the round step. The global model $\bar{x}^t$ is updated by aggregating the local gradients $g_i^t$ using the FedAvg algorithm. In FedAvg, each client updates its local model for $I$ steps: $x_i^{t+1} = x_i^t - \eta \nabla f_i(x_i^t)$, where $\eta$ is the learning rate. The global model is updated every $I$ steps by aggregating the local models as $\bar{x}^{t} = \frac{1}{N} \sum_{i=1}^{N} x_i^{t}$. For steps where $t \mod I \neq 0$, global averaging is not performed; however, these intermediate local updates are still included in the theoretical analysis for convergence derivation. For convenience, we define $r = \lfloor \frac{t}{I} \rfloor$ as the global round. Let \( G_i^{t,r} \) denote the segmented gradient matrix at the \( t \)-th step (which belongs to global round \(r\)). The aggregate segmented gradient matrix of client \( i \) in round \( r \) is then defined as \( G_i^{*,r} = \sum_{\tau = Ir}^{Ir + I} G_i^{\tau,r} \). In our method, the reconstructed gradient is denoted by $ \hat{g} $, and the error is defined as $ e = g - \hat{g} $, which satisfies $ \langle e, \hat{g} \rangle = 0$. 
The assumptions for the analysis are as follows: \\
\textbf{Assumption 1.} (Lipschitz Smooth) \textit{Each function $f_i(x)$ is smooth with modulus $L$, such that} $\| \nabla f_i(x) - \nabla f_i(y) \| \leq L \| x - y \|, \forall x, y, \forall i.$\\
\textbf{Assumption 2.} (Bounded Variance) \textit{There exist constants $\sigma > 0$ such that} $\mathbb{E} [\|g_i(x) - \nabla f_i(x)\|^2] \leq \sigma^2,\forall x, \forall i.$ \\
\textbf{Assumption 3.} (Bounded Gradients) \textit{There exist constants $\rho > 0$ such that} $\mathbb{E} [\|g_i(x)\|^2] \leq \rho^2, \forall x, \forall i.$ \\
\textbf{Assumption 4.} (Gradient Correlation) \textit{
	Let \(M_i^{r}\) be an orthonormal basis of the rank-\(k\) subspace spanned by the top-\(k\) left singular vectors of \(G_i^{*,r-1}\). 
	Define the subspace concentration ratio as
	\[
	\chi_k(G_i^{t,r}, G_i^{*,r-1}) = \frac{\|{M_i^r}^\top G_i^{t,r}\|}{\|G_i^{t,r}\|}.
	\]
	We assume that the expected subspace concentration satisfies
	\(
	\mathbb{E}\big[\chi_k^2(G_i^{t,r}, G_i^{*,r-1})\big] \ge \delta^2,
	\)
	where \(0 \le \delta \le 1\) quantifies the minimum correlation between the current gradient matrix and the subspace spanned by the previous round's gradients.
}

Assumption 1-3 are standard in the FL setting, providing the foundation for convergence analysis~\cite{wang2023svdfed,10229032,yu2019parallel}. Assumption 4, inspired by~\cite{gur2018gradient}, formalizes the empirical observation that the gradient matrix of the current round is concentrated in a low-dimensional subspace learned from past rounds. Moreover, Assumption 4 is consistent with our method: by incrementally updating basis vectors, GradESTC ensures that the orthonormal basis \(M\) remains a good approximation of the top-\(k\) subspace of \(G_i^{*,r-1}\).
We also define the inter-client gradient error correlation as
\(
\tau \triangleq \max_{i \neq j} \mathbb{E}[\langle e_i^{t}, e_j^{t} \rangle],
\)
which measures the degree of correlation between the compression-induced errors of different clients. By definition, $\tau \in [0, (1-\delta^2)\rho^2]$, where the upper bound corresponds to the case of highly correlated client errors.
A small value of $\tau$ indicates that the clients' errors are nearly independent, whereas a large $\tau$ implies strong correlation among client errors.

\paragraph{Main Results}
The convergence proofs of the following theorems are based on these assumptions and are inspired by the analytical framework in \cite{yu2019parallel}. We have the following main results:

\begin{theorem}	\label{thm:e4error}
	Under Assumption 4, the expected gradient reconstruction error in GradESTC is bounded as: 
	\[
	\mathbb{E}[\|e_i^{t}\|^2] = \mathbb{E}[\|G_i^{t,r} - {M_i^{r}}^\top G_i^{t,r}\|^2] \leq \big(1-\delta^2\big)\rho^2,
	\]
	and the average reconstruction error across all clients satisfies
	\[
	\mathbb{E}[\|\frac{1}{N}\sum_{i=1}^{N} e_i^{t}\|^2] \leq \frac{1}{N}\big((1 -\delta^2)\rho^2 + (N-1) \tau\big).
	\]

\end{theorem}

\begin{theorem}	\label{thm:convergence}
	Under Assumptions 1-3 and Theorem~\ref{thm:e4error}, for all $T > 1$, the following inequality holds:
	\begin{align*}
		&\frac{1}{T} \sum_{t=1}^{T} \mathbb{E}[\|\nabla f(\overline{x}^{t-1})\|^2] \leq \frac{4}{\eta T} (f(\overline{x}^{0}) - f^*)+ \frac{4 \eta \sigma^2}{N} \nonumber \\
		& + 16L^2 \eta^2 I^2 \rho^2 + \frac{4(1 + L\eta)}{N}\bigg((1 -\delta^2)\rho^2 + (N-1) \tau\bigg).
	\end{align*}
\end{theorem} 
The detailed proofs of Theorems~\ref{thm:e4error} and~\ref{thm:convergence} are provided in the Appendix. Following Theorem~\ref{thm:convergence}, we present a corollary as follows:


\begin{corollary}
	\label{cor:convergence}
	Under the same conditions as Theorem~\ref{thm:convergence}, assume that $T \geq N$. 
	If we choose $\eta = \frac{\sqrt{N}}{L\sqrt{T}}$ and $I \leq \frac{T^{1/4}}{N^{3/4}}$, then we have 
	\begin{align*}
		&\frac{1}{T} \sum_{t=1}^{T} \mathbb{E}[\|\nabla f(\overline{x}^{t-1})\|^2] \leq \frac{4L}{\sqrt{N T}} (f(\overline{x}^{0}) - f^*)+ \frac{4 \sigma^2}{L\sqrt{NT}} \nonumber \\
		& + \frac{16\rho^2}{\sqrt{NT}}  + \left(\frac{4}{N}+\frac{4}{\sqrt{NT}}\right)\bigg((1 -\delta^2)\rho^2 + (N-1) \tau\bigg).
	\end{align*}
\end{corollary}
\paragraph{Discussion} Theorem~\ref{thm:convergence} and Corollary~\ref{cor:convergence} indicates that GradESTC converges to a neighborhood of a stationary point rather than exactly to zero. The size of this steady-state neighborhood is determined by the last term in the bound, which depends on \((1-\delta^2)\rho^2\) and \(\tau\).

When \(\tau \ll (1-\delta^2)\rho^2\), the inter-client errors are nearly independent. In this case, the term \(\frac{4}{N}(1-\delta^2)\rho^2\) diminishes as \(N\) increases, while the term \(\frac{4}{\sqrt{NT}}(1-\delta^2)\rho^2\) diminishes with \(T\). Consequently, the convergence neighborhood shrinks as the number of clients grows, and GradESTC achieves a convergence rate close to the standard FedAvg rate, \(O(\frac{1}{\sqrt{NT}})\). However, if \(\tau\) is comparable to \((1-\delta^2)\rho^2\), the residuals across clients are strongly correlated. In this scenario, the sum \((1-\delta^2)\rho^2 + (N-1)\tau\) scales approximately linearly with \(N\). As a result, increasing the number of clients no longer reduces the steady-state error, and the convergence rate degrades significantly.

Intuitively, $\tau$ reflects the degree of similarity among clients' compression errors. Under relatively IID data distributions, local gradients tend to align in similar directions, often leading to more correlated reconstruction errors and thus a larger $\tau$. Conversely, in highly non-IID scenarios, local gradients usually diverge across clients, resulting in less correlated or nearly orthogonal reconstruction errors and a smaller $\tau$.


\section{Experiments} \label{sec:experiments}  
This section evaluates the effectiveness of each component of GradESTC and assesses the overall performance of the algorithm in terms of communication overhead, model accuracy, and convergence speed. Experiments are conducted under both IID and non-IID settings, which are common in FL, to demonstrate the generalizability of GradESTC. All experiments are implemented in the PyTorch framework and run on a GeForce RTX 3090 Ti GPU. GradESTC\footnote{Code: \url{https://github.com/zsl503/GradESTC}} is built upon FedAvg, a widely adopted FL algorithm.

\subsection{Comparison Experiment} \label{sec:comparison}  
\begin{table} [!hb]
    \begin{center}
    \caption{Experimental Design.}
    \label{tab:experiment_design}
    \begin{tabular}{llrrr}
        \hline
        Dataset & Model & Param Size & Rounds & BS \\
        \hline
        MNIST~\cite{deng2012mnist}   & LeNet5~\cite{lecun2002gradient}    & 0.26 MB     & 100   & 32 \\
        CIFAR-10~\cite{krizhevsky2009learning} & ResNet18~\cite{he2016deep}   & 42.65 MB   & 100 & 32 \\
        CIFAR-100~\cite{krizhevsky2009learning} & AlexNet~\cite{krizhevsky2017imagenet}   & 217.61 MB    & 100 & 32 \\
        \hline
    \end{tabular}
	\end{center}
\end{table}
To evaluate the effectiveness of our method, we conduct comparison experiments on three widely used image classification datasets: MNIST, CIFAR-10, and CIFAR-100. The MNIST dataset contains 60,000 grayscale images of handwritten digits for training and 10,000 images for testing. The CIFAR-10 comprises 50,000 training images and 10,000 test images spanning 10 classes of natural scenes. The CIFAR-100 increases the classification challenge by extending to 100 classes, each containing 500 training images and 100 test images.

\paragraph{Experiment Setup} 
We employ three representative models with varying levels of complexity: LeNet5, ResNet18, and AlexNet, as summarized in Table~\ref{tab:experiment_design}. For each dataset, experiments are conducted under three data distribution scenarios: IID and non-IID with Dirichlet parameters $\alpha = 0.5$ and $\alpha = 0.1$, where $\alpha$ controls the degree of heterogeneity among clients. The "Rounds" column in Table~\ref{tab:experiment_design} indicates the number of global training rounds (i.e., aggregation steps) corresponding to each setting. All algorithms are trained with a learning rate of 0.01 using cross-entropy loss. We simulate a FL environment with 10 clients, all participating in every global round, and each client performing one local epoch of training on its private data before aggregation.

\begin{table*}[htbp]
\begin{center}
	\caption{Comparison Results in Different Datasets and Distributions.}
	\label{tab:All_results}
	\begin{tabular}{
		@{}
		l
		l 
		*{3}{r} 
		*{3}{r} 
		*{3}{r} @{} 
	}
		\hline
		\multirow{2}{*}{\makecell[c]{Distribution}} & \multirow{2}{*}{\makecell[c]{Method}} & \multicolumn{3}{c}{MNIST-LeNet5} & \multicolumn{3}{c}{CIFAR10-ResNet18} & \multicolumn{3}{c}{CIFAR100-AlexNet} \\ 
		\cline{3-11} & 
		& \makecell[c]{Uplink at \\ Threshold} & \makecell[c]{Total \\ Uplink} & \makecell[c]{Best \\ Accuracy} 
		& \makecell[c]{Uplink at \\ Threshold} & \makecell[c]{Total \\ Uplink} & \makecell[c]{Best \\ Accuracy} 
		& \makecell[c]{Uplink at \\ Threshold} & \makecell[c]{Total \\ Uplink} & \makecell[c]{Best \\ Accuracy}  \\
		\hline
		\multirow{5}{*}{\makecell[c]{IID}}
			& \multirow{1}{*}{FedAvg} 
				& 0.0299  & 0.1658 & *{98.98}  \quad
				& 10.8396  & 41.6908& \underline{83.11}  \quad 
				& 57.7672 & 213.9527 & 55.96  \quad   \\ 
			& \multirow{1}{*}{Top-k} 
				& - & 0.0500 & 97.13   \quad
				& 5.5037 & 12.5084 & 82.57   \quad
				& 32.0891 & 63.5361 & 56.37  \quad   \\
			& \multirow{1}{*}{FedPAQ} 
				& 0.0315  & 0.0417  & \underline{99.01}   \quad
				& 2.8215  & 10.4501 & 83.05  \quad 
				& 15.5268 & 53.5406 & 55.88  \quad  \\
			& \multirow{1}{*}{SVDFed} 
				& *{0.0111}  & \underline{0.0117}  & 98.27 \quad  
				& *{2.2682}  & \underline{4.6891}  & 82.66 \quad 
				& *{10.7562} & \underline{15.5282} & 56.46 \quad \\ 	
			& \multirow{1}{*}{FedQClip} 
				& 0.0113  & 0.0417  & 98.89  \quad   
				& 3.5530  & 10.4501 & *{83.09}  \quad 
				& 14.9914 & 53.5406 & *{57.57}  \quad\\ 	 	
			& \multirow{1}{*}{GradESTC}
				& \underline{0.0058} & *{0.0196}  & 98.75   \quad
				& \underline{1.4419} & *{5.2143}  & {83.04}   \quad
				& \underline{7.3583} & *{20.4534} & \underline{58.89}   \quad \\
		\hline
		\multirow{5}{*}{\makecell[c]{Non-IID \\ $\alpha=0.5$}}
			& \multirow{1}{*}{FedAvg} 
				& 0.0249  & 0.1658   & *{98.72}    \quad
				& 15.0087 & 41.6908  & 78.18   \quad
				& 77.0230 & 213.9527 & 54.43   \quad \\
			& \multirow{1}{*}{Top-k} 
				& - & 0.0500  & 96.54   \quad
				& 13.8480 & 17.9844 & 78.11   \quad
				& 32.7309 & 64.1778 & 54.75   \quad \\
			& \multirow{1}{*}{FedPAQ} 
				& 0.0063  & 0.0417  & \underline{98.74}   \quad
				& 5.2250  & 10.4501 & 78.57   \quad
				& 17.6684 & 53.5406 & *{54.96}  \quad   \\
			& \multirow{1}{*}{SVDFed} 
				& *{0.0061}  & \underline{0.0102}  & 98.22   \quad
				& 3.6799  & \underline{4.7770}  & 78.35   \quad
				& *{13.4150} & \underline{15.5282} & 54.62   \quad \\ 	 
			& \multirow{1}{*}{FedQClip} 
				& 0.0104  & 0.0417  & 98.46   \quad
				& *{3.2395}  & 10.4501  & *{78.89}   \quad
				& 24.0933 & 53.5406 & 54.58   \quad \\ 		
			& \multirow{1}{*}{GradESTC} 
				& \underline{0.0036} & *{0.0175}  & {98.57}   \quad
				& \underline{1.6765} & *{5.0959}  & \underline{79.28}   \quad
				& \underline{8.5952} & *{20.5002} & \underline{57.06}   \quad \\
		\hline
		\multirow{5}{*}{\makecell[c]{ Non-IID \\ $\alpha=0.1$}}
			& \multirow{1}{*}{FedAvg}
				& 0.0547  & 0.1658   & *{97.79}   \quad
				& 8.3382  & 41.6908  & 64.91   \quad
				& 66.3253 & 213.9527 & 52.13   \quad \\
			& \multirow{1}{*}{Top-k} 
				& - & 0.0500  & 94.40   \quad
				& 8.4527  & 17.9844 & 64.25   \quad
				& 32.7308 & 64.1778 & 50.51   \quad \\
			& \multirow{1}{*}{FedPAQ}
				& 0.0133  & 0.0417   & \underline{97.80}   \quad
				& 2.0900  & 10.4501  & 64.72   \quad
				& 14.9914 & 53.5406  & 52.03   \quad \\
			& \multirow{1}{*}{SVDFed} 
				& *{0.0081}  & \underline{0.0102}  & 96.20 \quad
				& *{1.4534}  & \underline{4.7770}  & 65.51 \quad
				& *{12.9593} & \underline{14.1220} & 49.81 \quad\\
			& \multirow{1}{*}{FedQClip}
				& 0.0108  & 0.0417  & 98.18   \quad
				& 2.8215  & 10.4501  & \underline{66.72}   \quad
				& 19.8100 & 53.5406  & *{52.34}  \quad \\ 	
			& \multirow{1}{*}{GradESTC} 
				& \underline{0.0036}  & *{0.0170}  & {97.50}   \quad
				& \underline{1.1586}  & *{5.0370}  & *{66.00}   \quad
				& \underline{8.4027}  & *{20.4798} & \underline{53.49}   \quad \\
		\hline
	\end{tabular}		
\end{center}
\end{table*}

\begin{figure*}[htbp]
	\centering
	\begin{minipage}{1\linewidth}
		\centering
		\includegraphics[width=0.9\linewidth]{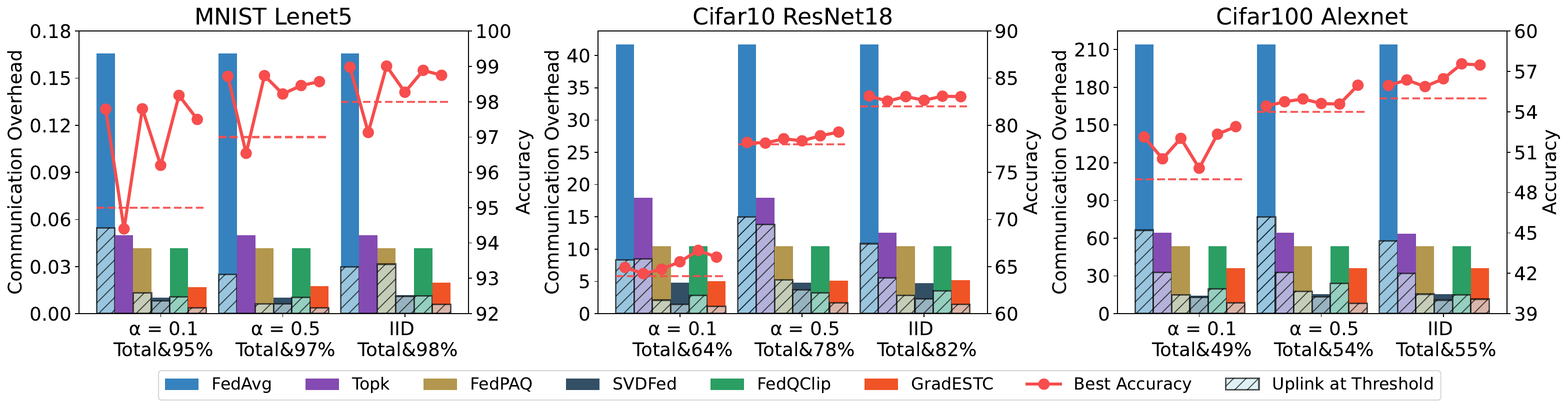}
	\end{minipage}
	\caption{Experimental Results of Different Algorithms. Solid bars indicate total uplink overhead, diagonal bars show overhead to reach $x$\% accuracy, red dots mark the peak accuracy achieved, and the red horizontal dotted line signifies the specified convergence accuracy threshold.}
	\label{fig:bar_results}
\end{figure*}

In the experiments, GradESTC is compared to baseline methods in terms of communication overhead, model accuracy, and convergence speed. Several approaches are evaluated, including Top-k~\cite{NEURIPS2018_b440509a}, which reduces transmitted bits through sparsification; FedPAQ~\cite{pmlr-v108-reisizadeh20a}, which reduces transmitted bits via quantization and periodic updates; SVDFed~\cite{wang2023svdfed}, a recent approach based on SVD that exploits gradient correlation across clients and rounds; and FedQClip~\cite{qu2024fedqclip}, which adopts gradient clipping and quantization to control local update magnitude and reduce communication overhead. For the Top-k algorithm, the $k$ value is set to 10, except for the CIFAR-10 non-IID setting where $k=20$ is used to ensure convergence. For FedPAQ, the quantization level is fixed at 8 (reducing the parameter size to approximately $1/4$ of its original 32-bit representation). For SVDFed, the hyperparameter $\gamma$ is selected based on empirical performance: $\gamma=8$ is used for both the MNIST and CIFAR-10 datasets under all data distributions, while $\gamma=6$ is used for CIFAR-100 across all settings. For FedQClip, $\eta_c = \eta_s = 0.01$ across all datasets. The clipped coefficient $(\gamma_c, \gamma_s)$ are set to (100, 10,000) for MNIST, (100, 300,000) for CIFAR-10, and (150, 15,000) for CIFAR-100.

\begin{figure*}[!t]
	\centering
	\begin{minipage}{1.0\linewidth}
		\centering
		\includegraphics[width=0.25\linewidth]{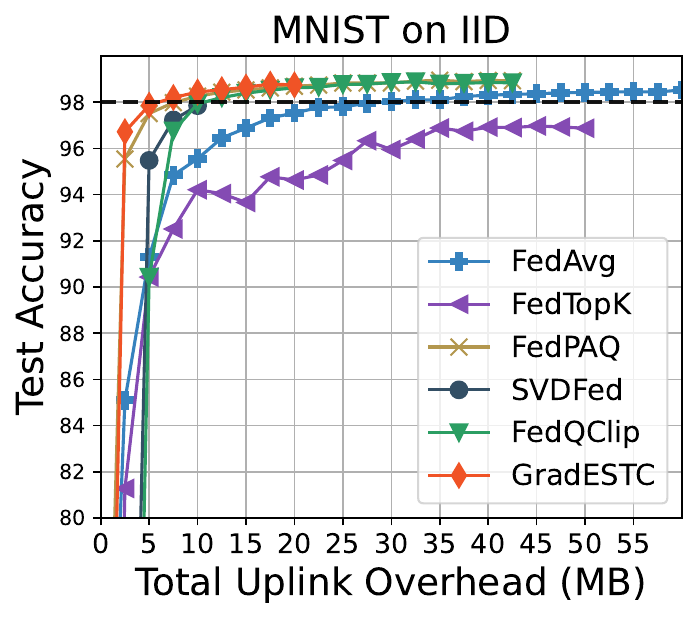}
		\includegraphics[width=0.25\linewidth]{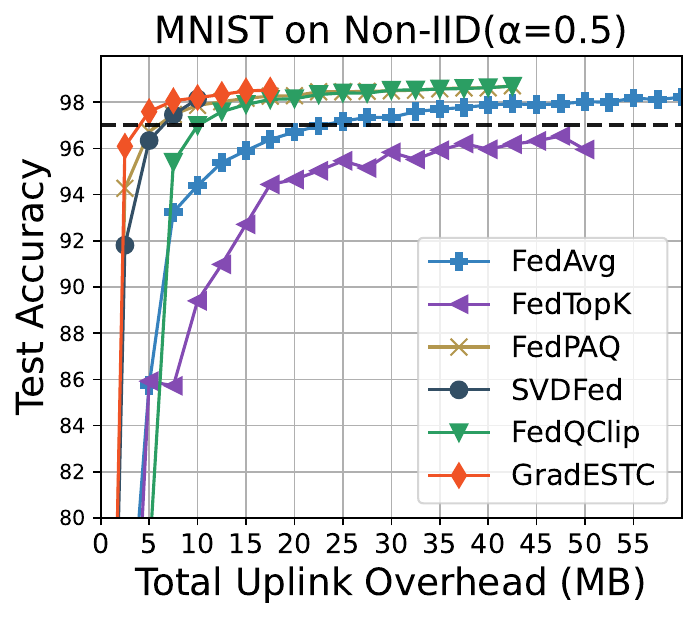}
		\includegraphics[width=0.25\linewidth]{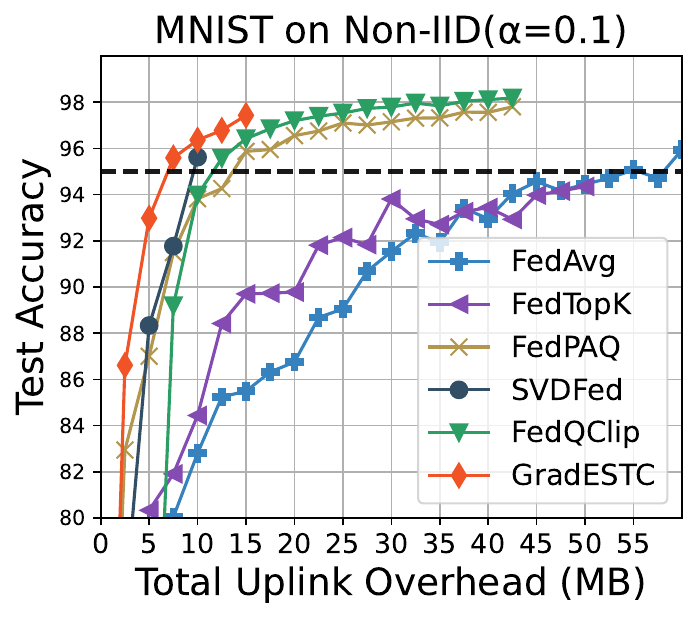}
	\end{minipage}
	\begin{minipage}{1.0\linewidth}
		\centering
		\includegraphics[width=0.25\linewidth]{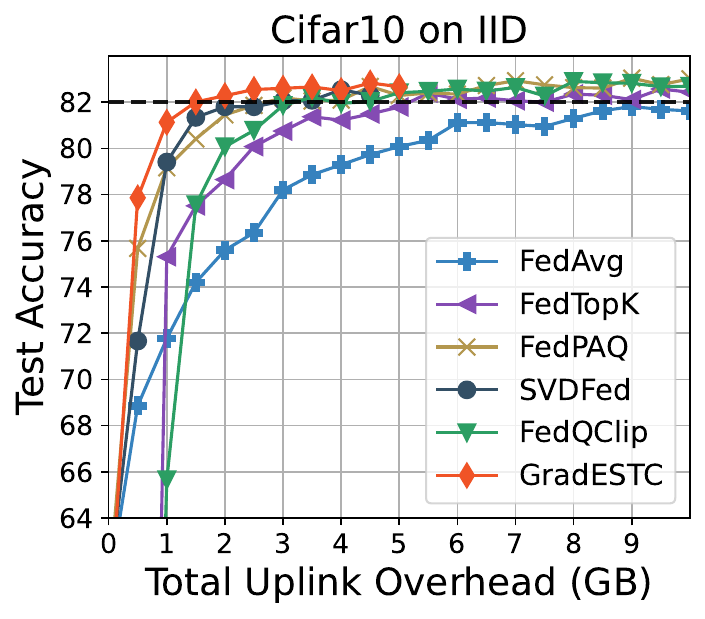}
		\includegraphics[width=0.25\linewidth]{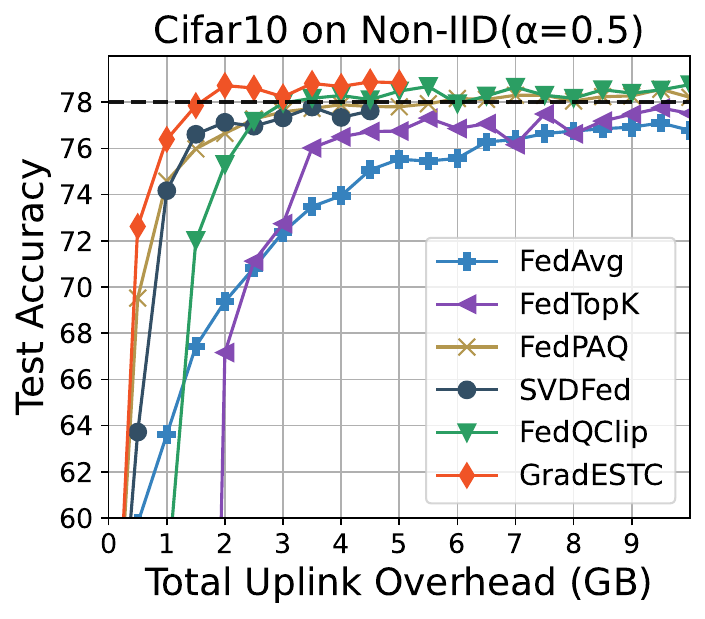}
		\includegraphics[width=0.25\linewidth]{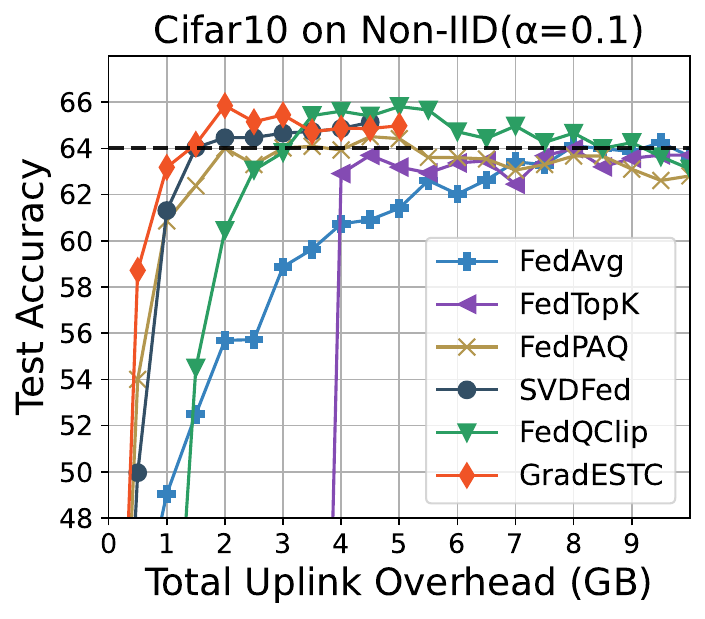}
	\end{minipage}
	\begin{minipage}{1.0\linewidth}
		\centering
		\includegraphics[width=0.25\linewidth]{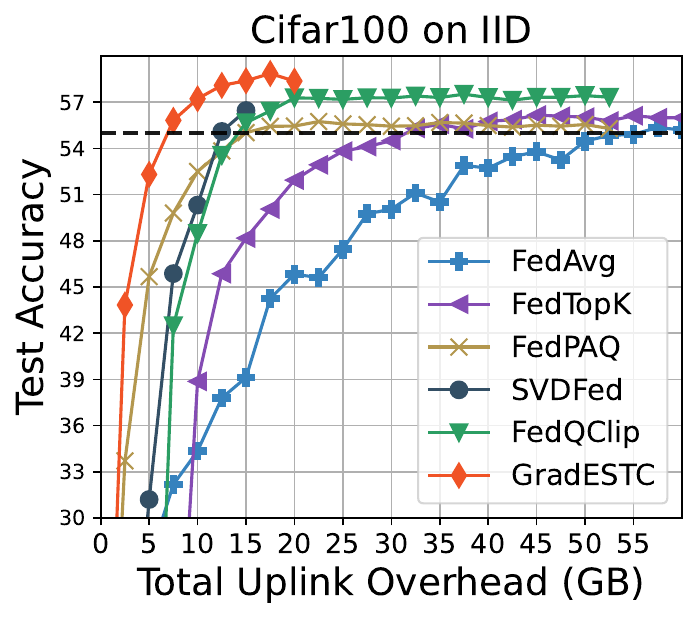}
		\includegraphics[width=0.25\linewidth]{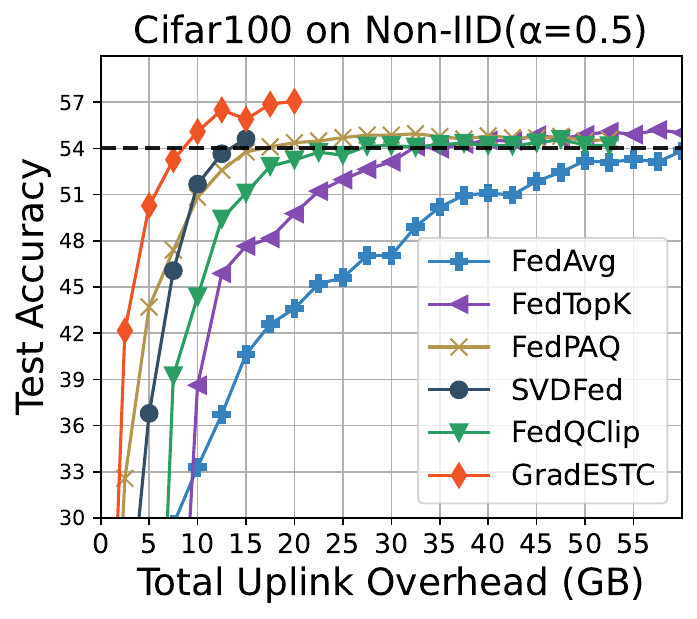}
		\includegraphics[width=0.25\linewidth]{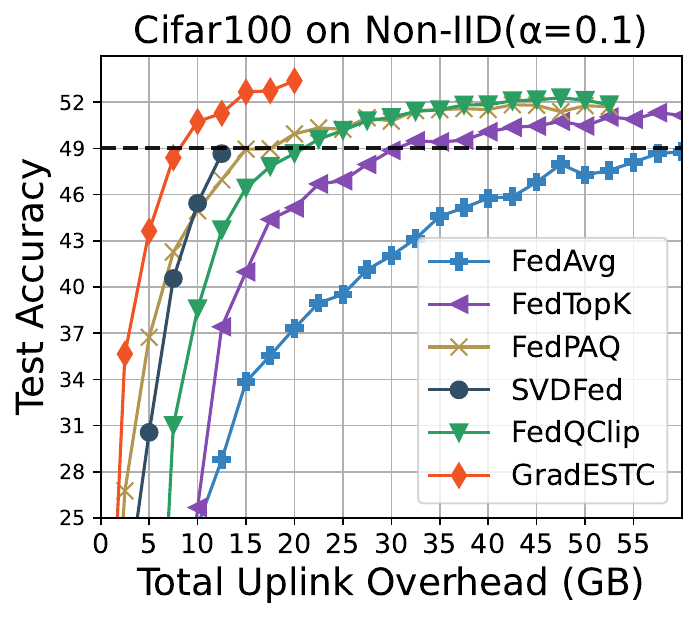}
	\end{minipage}
	\caption{Test Accuracy vs. Overhead. GradESTC achieves the highest accuracy with minimal communication overhead.}
	\label{fig:acc_plot_results}
\end{figure*}

\paragraph{GradESTC Settings}
Based on the observations in Section~\ref{sec:introduction} (Figures~\ref{fig:client_similarity} and~\ref{fig:layer_params}), GradESTC is applied for gradient compression only on layers that account for the majority of model parameters. In the LeNet5 model, the weights of the 'conv2', 'fc1', 'fc2', and 'classifier' layers are compressed, accounting for 99.0\% of the total model parameters. The gradient compression parameters $(k, l)$ are set to $(8, 160)$, $(16, 256)$, $(8, 120)$, and $(4, 28)$, respectively. 
In the ResNet18 model, the default PyTorch implementation is used. The weights of all 'conv1' and 'conv2' layers in stage 'layer3.0', 'layer3.1', 'layer4.0' and 'layer4.1' are compressed, accounting for 92.3\% of the total model parameters. The parameters are set to a fixed $k$ value of 32, and $l$ values of $1152, 2304, 768, 1536, 1024, 1536, 1536, 1536$, respectively.
In the AlexNet model, the weights of the 'conv3', 'conv4', 'conv5', 'fc1', and 'fc2' layers are compressed, accounting for 98.7\% of the total model parameters. The gradient compression parameters are set with a fixed $k=48$, and $l$ values of $288, 288, 256, 512, 1024$, respectively.
For each selected layer, only the weights with the largest proportion of parameters are compressed, while biases, batch normalization parameters, etc., are not compressed. 
           
Figure~\ref{fig:acc_plot_results} illustrates the relationship between accuracy and communication overhead across different datasets and distributions. The accuracy unit is percentage (\%), and the uplink overhead unit is GB. GradESTC consistently achieves the highest accuracy under the same communication budget. Figure~\ref{fig:round_plot_results} shows how accuracy evolves over communication rounds across different datasets. While GradESTC may exhibit slightly slower convergence during the mid-training phase, its final accuracy consistently matches or even surpasses that of uncompressed FedAvg.

\begin{figure*}[!t]
	\centering
	\begin{minipage}{1.0\linewidth}
		\centering
		\includegraphics[width=0.25\linewidth]{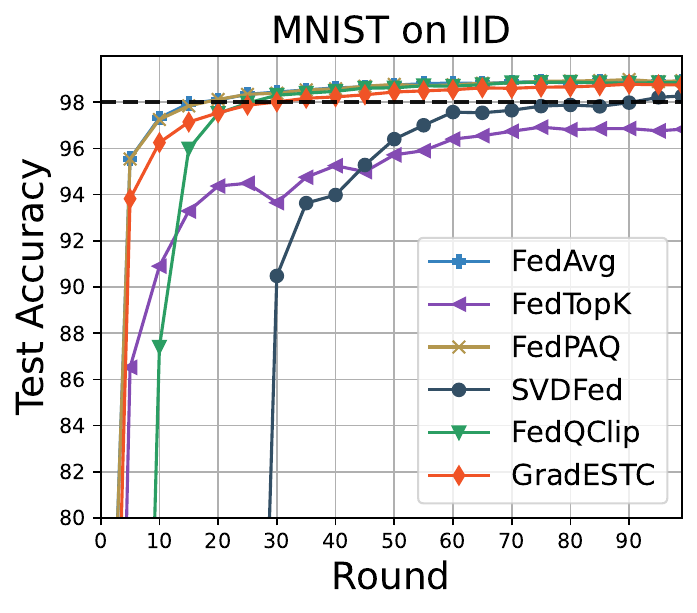}
		\includegraphics[width=0.25\linewidth]{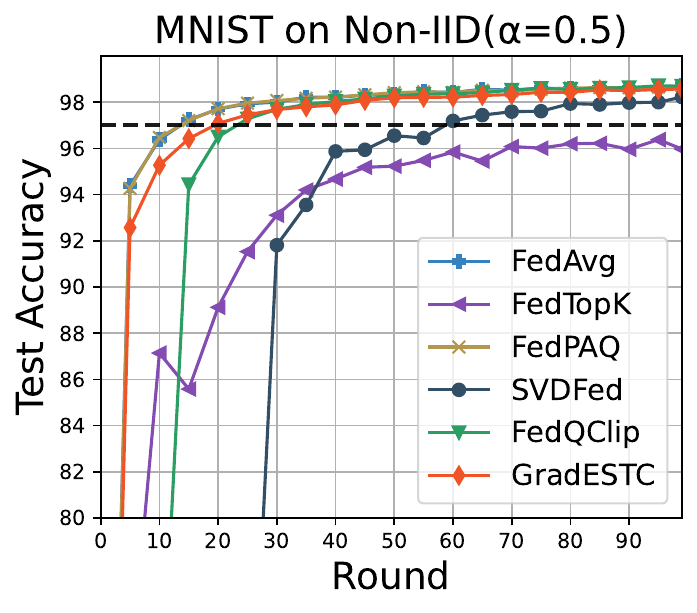}
		\includegraphics[width=0.25\linewidth]{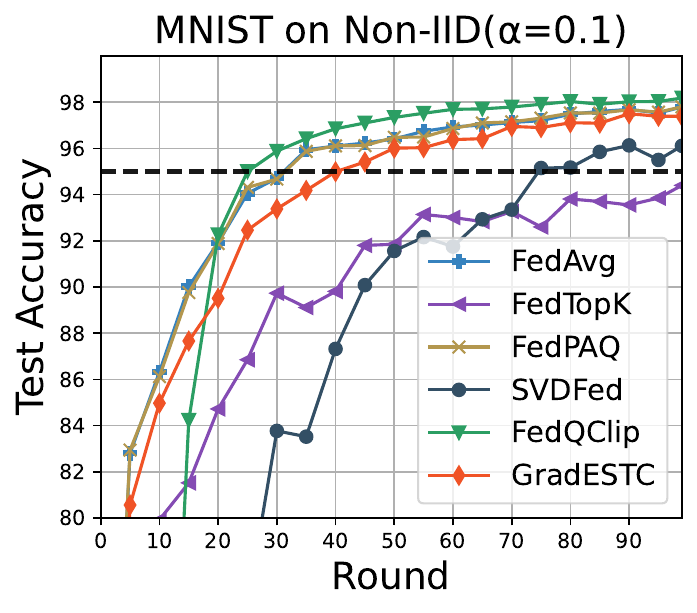}
	\end{minipage}
	\begin{minipage}{1.0\linewidth}
		\centering
		\includegraphics[width=0.25\linewidth]{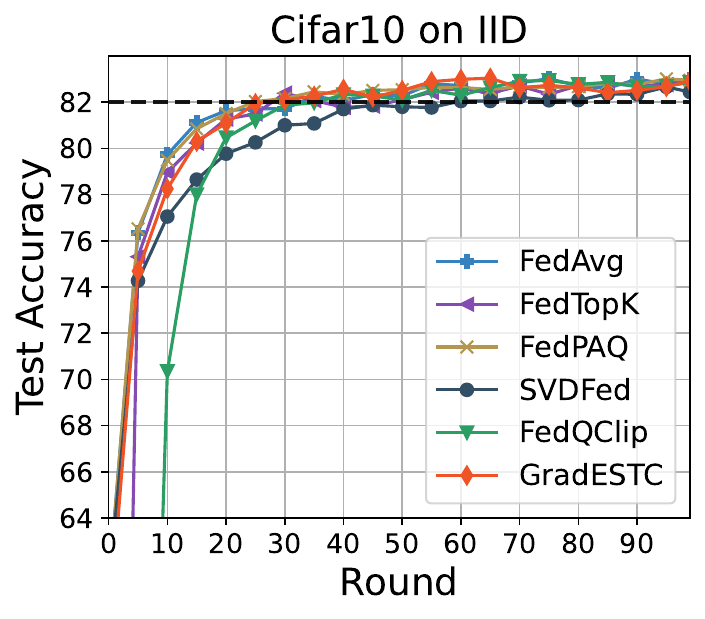}
		\includegraphics[width=0.25\linewidth]{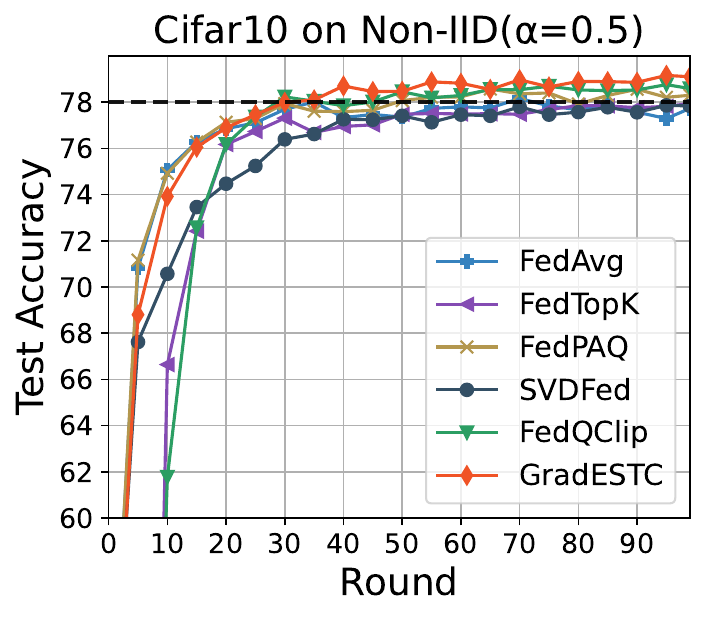}
		\includegraphics[width=0.25\linewidth]{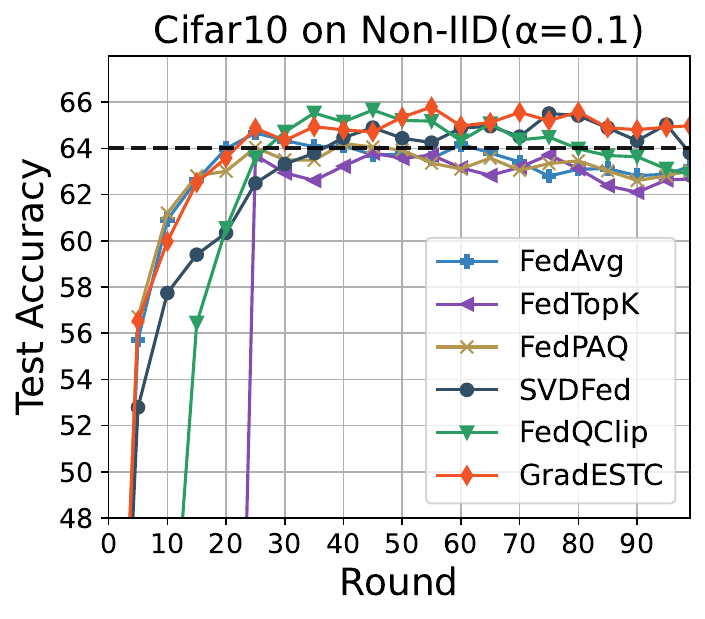}
	\end{minipage}
	\begin{minipage}{1.0\linewidth}
		\centering
		\includegraphics[width=0.25\linewidth]{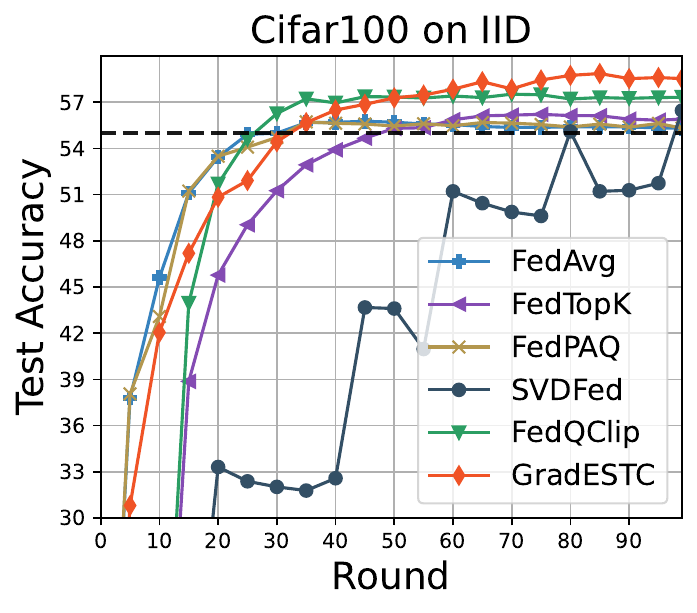}
		\includegraphics[width=0.25\linewidth]{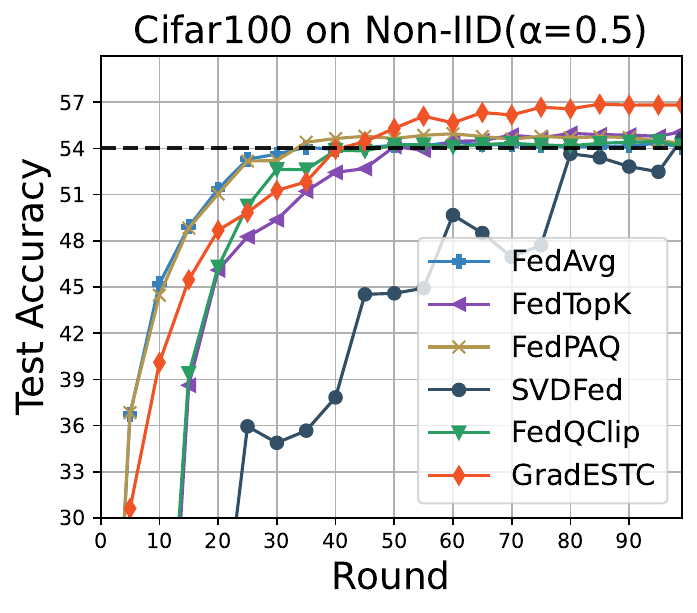}
		\includegraphics[width=0.25\linewidth]{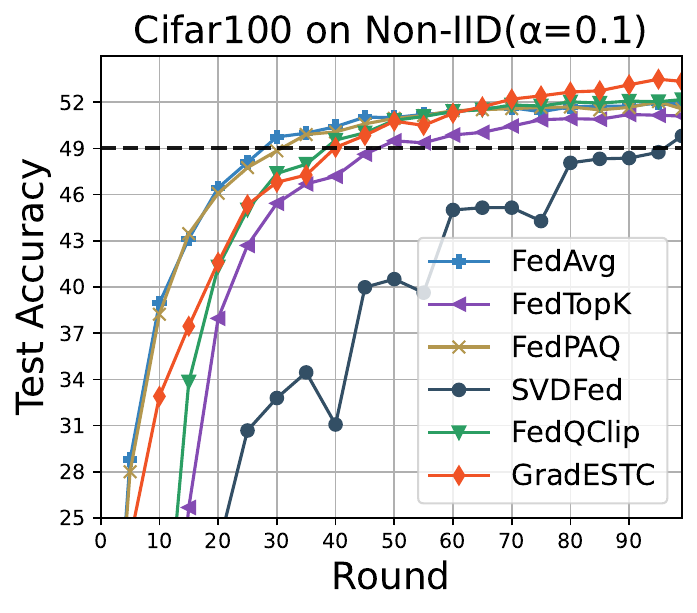}
	\end{minipage}
	\caption{Test Accuracy vs. Round. GradESTC achieves a convergence speed comparable to or even better than FedAvg.}
	\label{fig:round_plot_results}
\end{figure*}

\subsection{Experiment on large-scale client participation} \label{sec:large_client_participation}
To evaluate the performance of GradESTC in scenarios with large-scale client participation, we conducted experiments on the CIFAR-10 dataset using ResNet18. The number of clients was set to 50, and selected 20\% of clients to participate in each round. The results shown in Figure~\ref{fig:large_client_participation}, demonstrate GradESTC's scalability and effectiveness in large-scale FL scenarios.

\subsection{Impact of Local Epochs} \label{sec:local_epochs}
To analyze the impact of local epochs on model performance, we conducted experiments with different local epoch settings. The experiments were conducted on the CIFAR-10 dataset using the ResNet18 model, with local epochs set to 3, 5, and 7, while keeping other settings consistent with previous experiments. As shown in Figure~\ref{fig:local_epochs}, increasing the number of local epochs allows the basis vectors to more accurately capture the temporal correlation of gradients across local rounds within each client, thereby enabling more effective gradient compression.

\begin{figure}[!t]
	\centering
	\begin{minipage}{1.0\linewidth}
		\centering
		\includegraphics[width=0.49\linewidth]{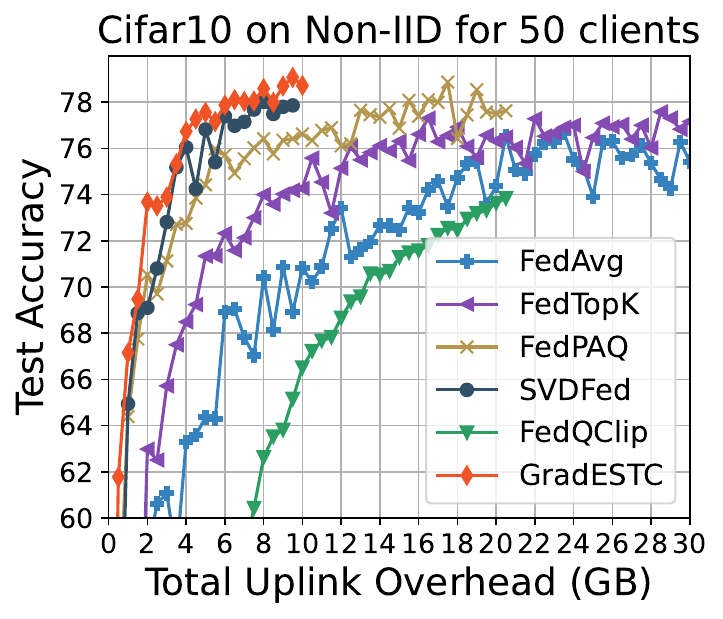}
		\includegraphics[width=0.49\linewidth]{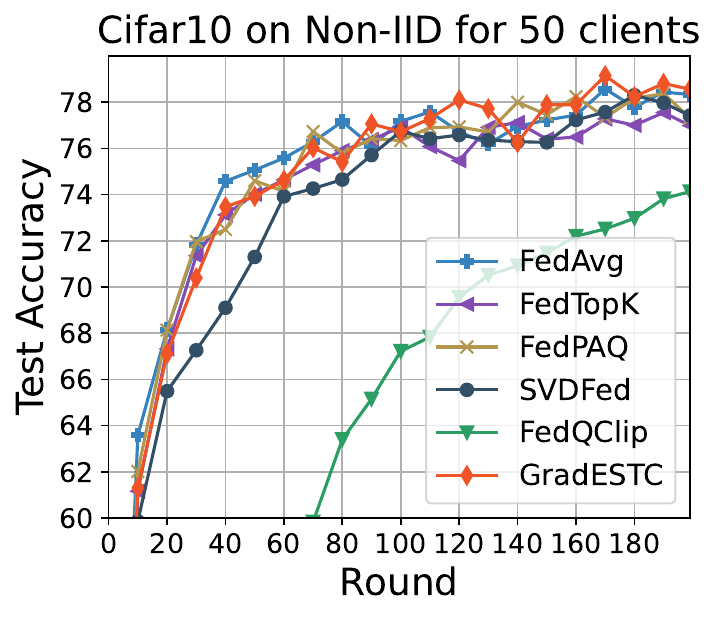}
	\end{minipage}
	\caption{Test Accuracy vs. Overhead and Round with 50 Clients.}
	\label{fig:large_client_participation}
\end{figure}

\begin{figure}[!t]
	\centering
	\begin{minipage}{1.0\linewidth}
		\centering
		\includegraphics[width=0.49\linewidth]{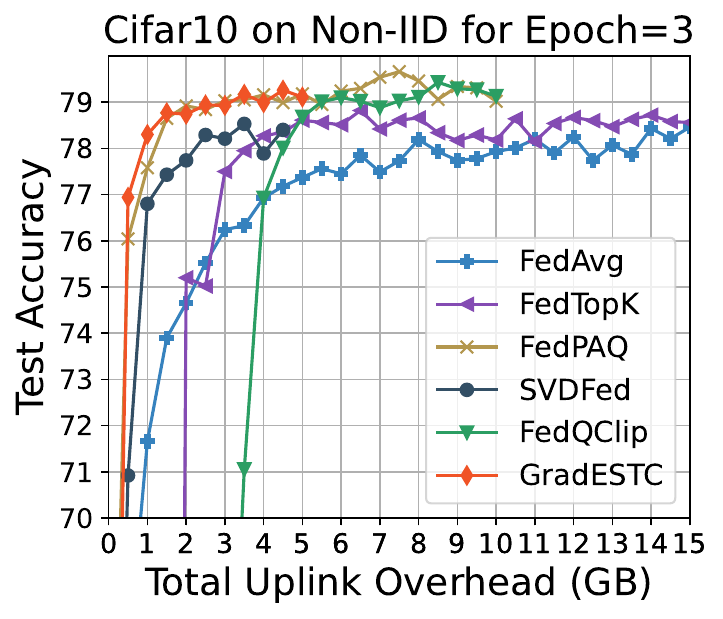}
		\includegraphics[width=0.49\linewidth]{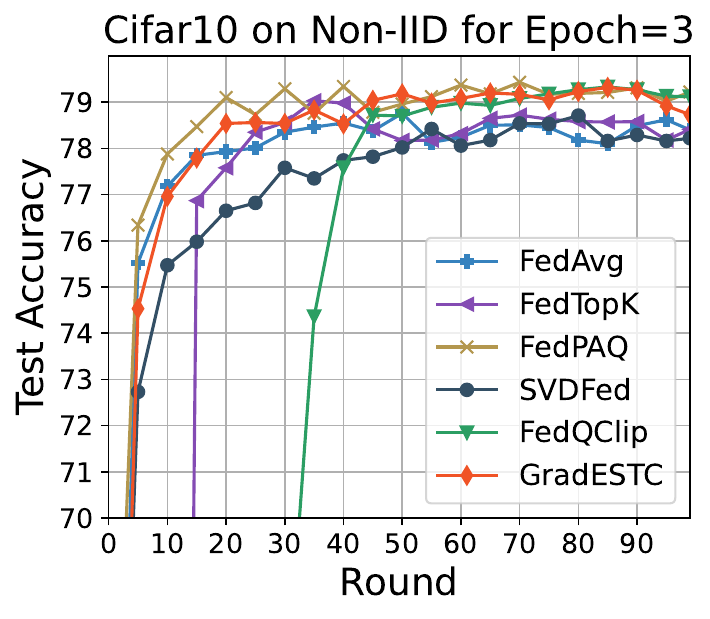}
	\end{minipage}
	\begin{minipage}{1.0\linewidth}
		\centering
		\includegraphics[width=0.49\linewidth]{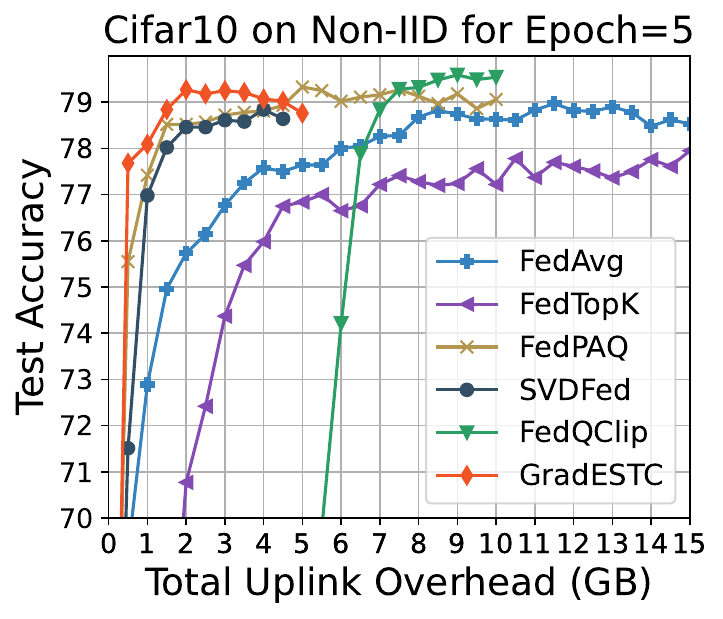}
		\includegraphics[width=0.49\linewidth]{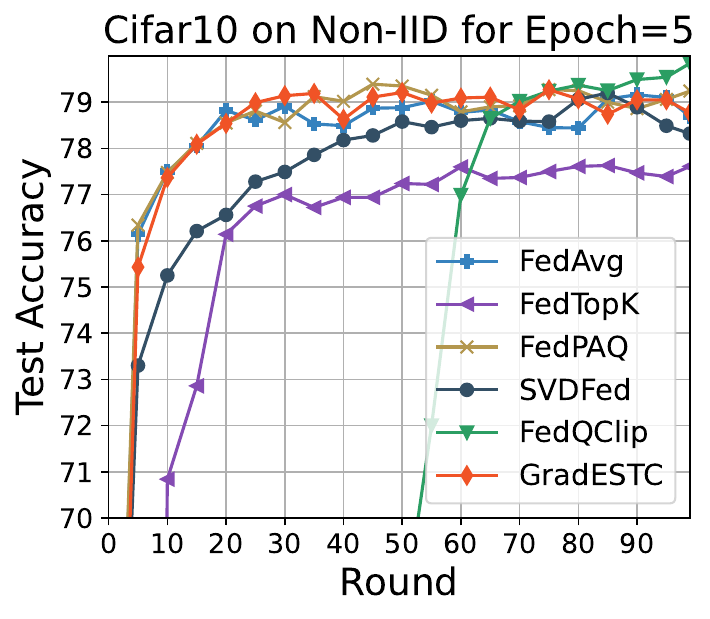}
	\end{minipage}
	\begin{minipage}{1.0\linewidth}
		\centering
		\includegraphics[width=0.49\linewidth]{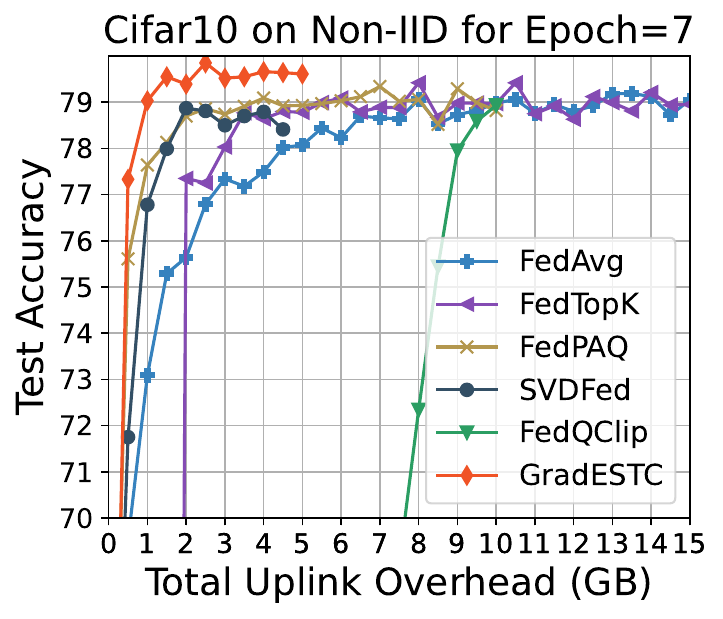}
		\includegraphics[width=0.49 \linewidth]{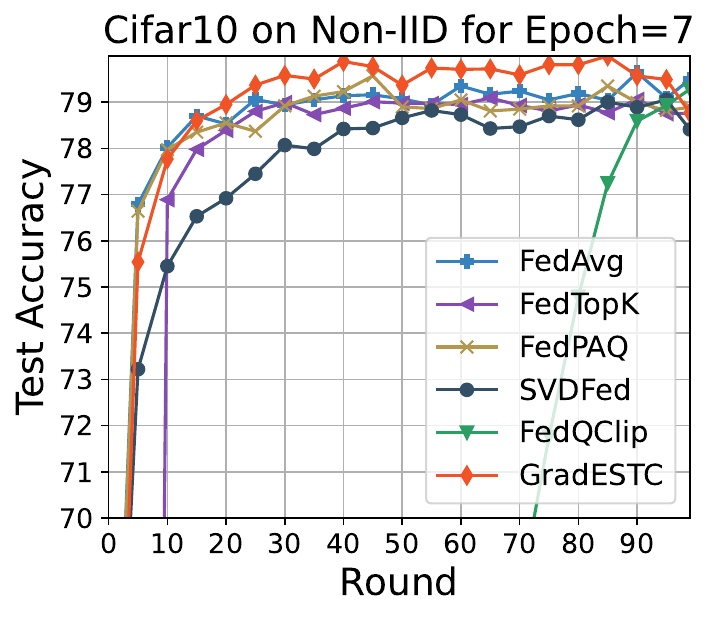}
	\end{minipage}
	\caption{Test Accuracy vs. Overhead and Round with Different Local Epochs.}
	\label{fig:local_epochs}
\end{figure}

\begin{figure}[!t]
	\centering
	\begin{minipage}{1.0\linewidth}
		\centering
		\includegraphics[width=0.49\linewidth]{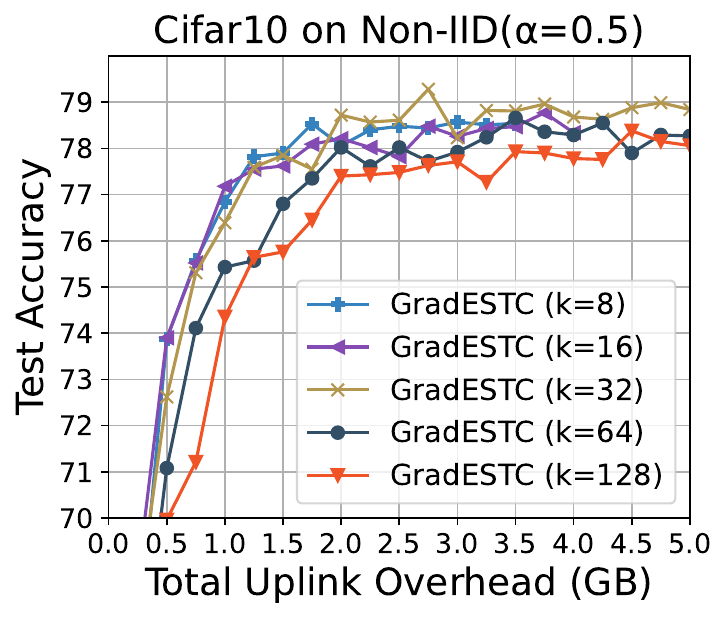}
		\includegraphics[width=0.49\linewidth]{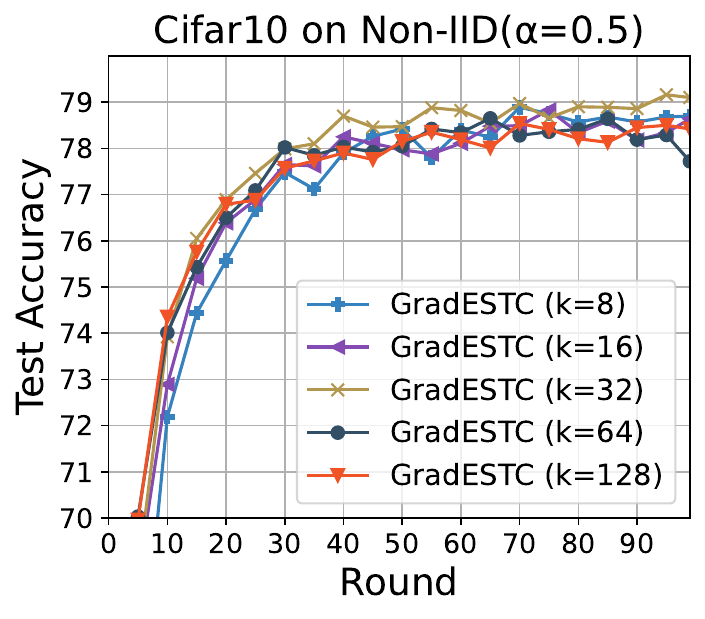}
	\end{minipage}
	\caption{Test Accuracy vs. Overhead and Round with Different $k$ Values.}
	\label{fig:k_sensitivity}
\end{figure}

\subsection{Impact of the $k$ for GradESTC} \label{sec:client_participation}
In our algorithm, the number of basis vectors $k$ influences both computational efficiency and gradient compression performance. A smaller $k$ reduces communication overhead but may limit the expressiveness of the gradients, potentially affecting model convergence and final accuracy. Conversely, a larger $k$ better captures gradient information but increases communication and computational overheads. Therefore, selecting an appropriate $k$ is a key design decision. However, since different layers of the model can adopt different $k$ values, it is difficult to precisely evaluate the impact of $k$ in general. To investigate the effect of $k$ in GradESTC, we conducted the sensitivity analysis on the CIFAR-10 dataset using the ResNet18 model, following the same experimental setup as before. This setup is valid because prior experiments have shown that using a uniform $k$ across all compressible layers in ResNet18 is effective. We varied $k$ among 8, 16, 32, 64, and 128, and observed its impact on model performance. The results are shown in Figure~\ref{fig:k_sensitivity}.


We found that the value of $k$ has a limited impact on communication overhead and model performance overall. However, an excessively large $k$ leads to increased communication overhead (e.g., when $k=128$, the trade-off between accuracy and overhead is the worst, as shown in the left plot of Figure~\ref{fig:k_sensitivity}), while a very small $k$ slows down convergence (e.g., when $k=8$, the convergence rate is noticeably slower in the early stages, as shown in the right plot of Figure~\ref{fig:k_sensitivity}). This is because a large $k$ introduces a greater number of combination coefficients to be transmitted, increasing communication overhead, whereas a small $k$ limits the expressiveness of the basis vectors, reducing the quality of gradient representation. Once $k$ reaches a sufficient value, the number of updated basis vectors per round is mainly determined by the dynamically adjusted $d$, making performance less sensitive to further increases in $k$. As a result, $k=64$ and $k=128$ show no significant difference in convergence speed or final model performance.

\subsection{Ablation Experiment} \label{sec:ablation}  
To validate the effectiveness of each component, we conducted a series of ablation experiment. Specifically, three variants were compared with the full GradESTC: GradESTC-first, GradESTC-all, and GradESTC-k. GradESTC-first initializes the basis vectors in the first round only without subsequent updates. GradESTC-all updates all basis vectors every round. GradESTC-k excludes dynamic adjustment and fixes $d$ as $k$. Experiments were conducted on the CIFAR-10 dataset. The sum of $d$ across all rounds and clients reflects computational overhead. As discussed in Section~\ref{sec:time_complexity}, with fixed $k, l, m$, execution efficiency primarily depends on $d$. Smaller sum of $d$ indicates lower computational overhead. In our setup, a single CIFAR-10 client performing full decomposition with $d = 256$ takes approximately 0.20 seconds, whereas a local training round of five epochs takes 8-10 seconds, making the computational overhead negligible.

\begin{table}
	\begin{center}
		\caption{Ablation Experiment Results.}
		\label{tab:ablation}
		\begin{tabular}{
			@{}
			l 
			*{4}{r}
			@{}}
			\hline
			Method & \makecell[c]{Best \\ Accuracy} & \makecell[c]{70\% \\ Uplink} & 
			\makecell[c]{Total \\ Uplink} & \makecell[c]{Sum of \\ $d$ values} \\
			\hline
			\makecell[l]{GradESTC-first} & 71.4600 & 2.9474 & 4.1462 & 2,560 \\
			\makecell[l]{GradESTC-all} & 75.0600 & 1.2420 & 5.4000 & 256,000 \\
			\makecell[l]{GradESTC-k} & 74.0800 & 1.1824 & 4.8447 & 256,000 \\
			GradESTC        & 75.6067 & 0.8892 & 4.8652 & 144,242 \\
			\hline
		\end{tabular}

	\end{center}
\end{table}
Results as shown in Table~\ref{tab:ablation} (where the accuracy unit is percentage and the uplink overhead unit is GB) indicate that the absence of basis vector updates severely impacts accuracy, where GradESTC-first yields the lowest accuracy among all algorithms shown in Table~\ref{tab:All_results}. This is because static basis vectors fail to effectively represent new gradients, thereby degrading training performance and convergence speed. GradESTC-all achieves accuracy near that of FedAvg, with a communication overhead of 5.40GB—representing a 10.99\% increase compared to the full GradESTC. It reduces FedAvg's communication overhead by 7.71×, demonstrating the efficacy of spatial correlation-based compression. However, updating all basis vectors in each round significantly increases the communication overhead compared to selective updating. GradESTC-k incrementally replaces basis vectors without adjusting $k$, reducing the basis vector uplink overhead by 0.55GB compared to GradESTC-all. However, it still incurs significant computational overhead. GradESTC-k and GradESTC-all use fixed $k$ values for SVD, whereas the full GradESTC dynamically adjusts $k$, resulting in a 43.66\% reduction in computational overhead compared to GradESTC-k, thereby demonstrating the effectiveness of the dynamic $d$ adjustment strategy.

\section{Conclusion}
In this paper, we proposed GradESTC, a novel compression framework that exploits spatio-temporal correlations in gradients to minimize communication overhead in FL. Extensive experiments demonstrate that GradESTC achieves the lowest communication overhead among all compared methods while maintaining accuracy and convergence speed comparable to uncompressed FedAvg. The proposed approach exhibits strong adaptability across different datasets, data distributions, and neural network architectures, making it especially effective in bandwidth-constrained environments. As future work, we plan to explore adaptive hyperparameter tuning to enhance the robustness and flexibility of GradESTC, investigate the exploitation of spatio-temporal correlations in the downlink phase to further improve communication efficiency, and incorporate error feedback mechanisms to reduce compression-induced residual errors.

{\appendix[Proof of Theorems]

\subsection{Proof of Theorem~\ref{thm:e4error}}
\textit{Theorem 1:} Under Assumption 4, the expected gradient reconstruction error in GradESTC is bounded as: 
\[
\mathbb{E}[\|e_i^{t}\|^2] = \mathbb{E}[\|G_i^{t,r} - {M_i^{r}}^\top G_i^{t,r}\|^2] \leq \big(1-\delta^2\big)\rho^2,
\]
and the average reconstruction error across all clients satisfies
\[
\mathbb{E}\Big[\|\frac{1}{N}\sum_{i=1}^{N} e_i^{t}\|^2\Big] \leq \frac{1}{N}\big((1 -\delta^2)\rho^2 + (N-1) \tau\big).
\]

\begin{proof}
By Assumption 4, the projection preserves at least a fraction \(\delta\) of the gradient norm. Hence, for each client:
\begin{align}
\mathbb{E}[\|e_i^{t}\|^2] &= \mathbb{E}[\|G_i^{t,r} - M_i^r (M_i^r)^\top G_i^{t,r}\|^2] \nonumber \\
&= \mathbb{E}[\|G_i^{t,r}\|^2 - \|(M_i^r)^\top G_i^{t,r}\|^2] \nonumber \\
&= \mathbb{E}[\|G_i^{t,r}\|^2 (1 - \chi_k^2(G_i^{t,r}, G_i^{*,r-1}) )] \nonumber \\
&\le \big(1-\delta^2\big)\rho^2,
\end{align}
where we used the fact that \(M_i^r\) is an orthonormal basis of the subspace, so \(\|M_i^r (M_i^r)^\top G_i^{t,r}\| = \|(M_i^r)^\top G_i^{t,r}\|\).
For the average reconstruction error across all clients, we expand the squared norm and use the assumption that the covariance between different clients' errors is bounded by \(\tau\):
\begin{align}
\mathbb{E}\Big[\|\frac{1}{N}\sum_{i=1}^{N} e_i^{t}\|^2\Big] &= \frac{1}{N^2} \mathbb{E}\Big[\|\sum_{i=1}^{N} e_i^{t}\|^2\Big] \nonumber \\
&= \frac{1}{N^2} \mathbb{E}\Big[\sum_{i=1}^{N} \|e_i^{t}\|^2 + \sum_{i \neq j} \langle e_i^{t}, e_j^{t} \rangle\Big] \nonumber \\
&\leq \frac{1}{N^2} (N(1 -\delta^2)\rho^2 + N(N-1) \tau) \nonumber \\
&= \frac{1}{N}\big((1 -\delta^2)\rho^2 + (N-1) \tau\big),
\end{align}
which proves the theorem~\ref{thm:e4error}.
\end{proof}

\subsection{Proof of Theorem~\ref{thm:convergence}}

To prove Theorem~\ref{thm:convergence}, we first establish a bound on the divergence between local and global models.
\begin{lemma} \label{lm:x2avg}
	Let $x_i^t$ be the local model of client $i$ at step $t$ and $\overline{x}^t$ be the global model at step $t$. 
	Assume that at the beginning of each round $r$, all clients start from the same global model and perform $I$ local updates. 
	Then the expectation of the squared distance between $x_i^t$ and $\overline{x}^t$ is bounded by
	\begin{align*}
		\mathbb{E}[\|x_i^t - \overline{x}^t\|^2] \leq 4\eta^2 I^2\rho^2.
	\end{align*}
\end{lemma}
\begin{proof}
	This result follows from bounding the accumulation of local updates and the difference between averaged global and local gradients:
	\begin{align}
		&\mathbb{E}[\|x_i^t - \overline{x}^t\|^2] = \mathbb{E}\Big[\|\eta \sum_{\tau=t_0(r)}^{t} \frac{1}{N} \sum_{j=1}^{N} g_j^\tau - \eta \sum_{\tau=t_0(r)}^{t} \hat{g}_i^\tau\|^2\Big] \nonumber \\
		&\leq 2\eta^2 \mathbb{E}\Big[\|\sum_{\tau=t_0(r)}^{t} \frac{1}{N} \sum_{j=1}^{N} g_j^\tau\|^2 + \|\sum_{\tau=t_0(r)}^{t} \hat{g}_i^\tau\|^2\Big] \nonumber \\
		&\leq 2\eta^2 I \mathbb{E}\Big[\sum_{\tau=t_0(r)}^{t} \frac{1}{N} \sum_{j=1}^{N} \|g_j^\tau\|^2 + \sum_{\tau=t_0(r)}^{t} \|\hat{g}_i^\tau\|^2\Big] \nonumber \\
		&\leq 4\eta^2 I^2\rho^2,
	\end{align}
	where we use $\|\hat{g}_i^\tau\|^2 \leq \rho^2$ as $\|\hat{g}_i^\tau\|^2 = \|g_i^\tau\|^2 - \|e_i^\tau\|^2 \leq \|g_i^\tau\|^2$. This completes the proof of Lemma~\ref{lm:x2avg}.
\end{proof}

\textit{Theorem 2:} Under Assumptions 1-3 and Theorem~\ref{thm:e4error}, for all $T > 1$, the following inequality holds:
\begin{align*}
	&\frac{1}{T} \sum_{t=1}^{T} \mathbb{E}[\|\nabla f(\overline{x}^{t-1})\|^2] \leq \frac{4}{\eta T} (f(\overline{x}^{0}) - f^*)+ \frac{4 \eta \sigma^2}{N} \nonumber \\
	& + 16L^2 \eta^2 I^2 \rho^2 + \frac{4(1 + L\eta)}{N}\bigg((1 -\delta^2)\rho^2 + (N-1) \tau\bigg).
\end{align*}
\begin{proof}
	By the \(L\)-smoothness of \(f\), we have
	\begin{align} \label{eq:convergence}
		&\mathbb{E}[f(\overline{x}^t)] \leq \mathbb{E}[f(\overline{x}^{t-1})] \stackrel{(a)}{+} \mathbb{E}[\langle \nabla f(\overline{x}^{t-1}), \overline{x}^t - \overline{x}^{t-1} \rangle]  \nonumber \\
		&\stackrel{(b)}{+} \frac{L}{2} \mathbb{E}[\|\overline{x}^t - \overline{x}^{t-1}\|^2].
	\end{align}
	Using the update $\overline{x}^t = \overline{x}^{t-1} - \eta \frac{1}{N} \sum_{i=1}^N \hat{g}_i^t$ with $\hat{g}_i^t = g_i^t - e_i^t$, we split the term (a) into the standard FedAvg term and the compression error term:
	\begin{align} \label{eq:convergence_a}
		&\mathbb{E}[\langle \nabla f(\mathbf{\overline{x}}^{t-1}), \overline{x}^t - \overline{x}^{t-1} \rangle] = -\eta \mathbb{E}\Big[\langle \nabla f(\overline{x}^{t-1}), \frac{1}{N} \sum_{i=1}^N \hat{g}_i^t \rangle \Big] \nonumber \\ 
		&= -\eta \mathbb{E}\Big[\langle \nabla f(\overline{x}^{t-1}), \frac{1}{N} \sum_{i=1}^N g_i^t \rangle - \langle \nabla f(\overline{x}^{t-1}), \frac{1}{N} \sum_{i=1}^N e_i^t \rangle\Big].
	\end{align}
	The first term is handled as in standard FedAvg analyses, we expand the inner product into squared norms and apply the unbiasedness and variance bound of stochastic gradients:
	\begin{align}
	&-\eta \mathbb{E}[\langle \nabla f(\overline{x}^{t-1}), \frac{1}{N} \sum_{i=1}^N g_i^t \rangle] = -\frac{\eta}{2} \mathbb{E}[\|\nabla f(\overline{x}^{t-1})\|^2] \nonumber \\
	& -\frac{\eta}{2N^2} \mathbb{E}\Big[\|\sum_{i=1}^N g_i^{t}\|^2\Big] + \frac{\eta}{2} \mathbb{E}\Big[\|\nabla f(\overline{x}^{t-1}) - \frac{1}{N} \sum_{i=1}^N g_i^{t}\|^2\Big] \nonumber \\
	&= -\frac{\eta}{2} \mathbb{E}[\|\nabla f(\overline{x}^{t-1})\|^2] -\frac{\eta}{2N^2} \mathbb{E}\Big[\|\sum_{i=1}^N g_i^{t}\|^2 \Big] \nonumber \\
	&+ \eta \mathbb{E}\Big[\|\nabla f(\overline{x}^{t-1}) - \frac{1}{N} \sum_{i=1}^N \nabla f_i(x_i^{t-1})\|^2\Big] \nonumber \\
	&+ \eta \mathbb{E}\Big[\|\frac{1}{N} \sum_{i=1}^N \big(\nabla f_i(x_i^{t-1}) - g_i^{t}\big)\|^2\Big] \nonumber \\
	&\leq -\frac{\eta}{2} \mathbb{E}[\|\nabla f(\overline{x}^{t-1})\|^2] -\frac{\eta}{2N^2} \mathbb{E}\Big[\|\sum_{i=1}^N g_i^{t}\|^2 \Big] \nonumber \\
	&+ \frac{L^2\eta }{N} \mathbb{E}\Big[\sum_{i=1}^N \|\overline{x}^{t-1} - x_i^{t-1}\|^2\Big] + \frac{\eta \sigma^2}{N}.
	\end{align}
	The inner product term induced by compression error in Equation~\ref{eq:convergence_a} is bounded using Young's inequality: for any vectors $a, b$ and any $\alpha>0$, $\langle a, b \rangle \le \frac{1}{2\alpha}\| a\|^2 + \frac{\alpha}{2}\|b\|^2$. Here we set $\alpha = 2$, yielding	
	\begin{align}
	&\eta\mathbb{E}\Big[\Big\langle \nabla f(\overline{x}^{t-1}), \frac{1}{N}\sum_{i=1}^N e_i^t\Big\rangle\Big] \nonumber \\
	&\le \eta\mathbb{E}\Big[\frac{1}{4}\|\nabla f(\overline{x}^{t-1})\|^2 + \| \frac{1}{N}\sum_{i=1}^N e_i^t\|^2\Big]  \nonumber\\
	&= \frac{\eta}{4}\mathbb{E}\big[\|\nabla f(\overline{x}^{t-1})\|^2\big] + \eta\mathbb{E}\Big[\| \frac{1}{N}\sum_{i=1}^N e_i^t\|^2\Big].
	\end{align}
	Next bound term (b) in Equation~\ref{eq:convergence}:   
	\begin{align}\label{eq:convergence_b}
	&\mathbb{E}\big[\|\overline{x}^t-\overline{x}^{t-1}\|^2\big]
	= \frac{\eta^2}{N^2}\mathbb{E}\Big[\Big\|\sum_{i=1}^N (g_i^t - e_i^t)\Big\|^2\Big] \nonumber\\
	&\le \frac{2\eta^2}{N^2}\mathbb{E}\Big[\Big\|\sum_{i=1}^N g_i^t\Big\|^2\Big] + 2\eta^2\mathbb{E}\Big[ \|\frac{1}{N}\sum_{i=1}^Ne_i^t \|^2\Big].
	\end{align}
	Combining the above bounds and using Lemma~\ref{lm:x2avg} for the local-global divergence, we obtain a one-step descent inequality:	
	\begin{align}
		&\mathbb{E}[f(\overline{x}^t)] \leq \mathbb{E}[f(\overline{x}^{t-1})] -\frac{\eta}{2} \mathbb{E}[\|\nabla f(\overline{x}^{t-1})\|^2] \nonumber\\
		&-\frac{\eta}{2N^2} \mathbb{E}\Big[\|\sum_{i=1}^N g_i^{t}\|^2 \Big] + \frac{L^2 \eta}{N}\mathbb{E}[\|\overline{x}^{t-1} - x_i^{t-1}\|^2] + \frac{\eta \sigma^2}{N}  \nonumber \\
		&+ \frac{\eta}{4}\mathbb{E}\big[\|\nabla f(\overline{x}^{t-1})\|^2\big] + \eta\mathbb{E}\Big[\|\frac{1}{N}\sum_{i=1}^Ne_i^t\|^2\Big] \nonumber \\
		&+ \frac{L\eta^2}{N^2}\mathbb{E}\Big[\Big\|\sum_{i=1}^N g_i^{t}\Big\|^2\Big] + L\eta^2\mathbb{E}\Big[\|\frac{1}{N}\sum_{i=1}^N e_i^t\|^2\Big].
	\end{align}
	Rearranging terms gives:
	\begin{align}
		&\mathbb{E}[f(\overline{x}^t)] \leq \mathbb{E}[f(\overline{x}^{t-1})] -\frac{\eta}{4} \mathbb{E}[\|\nabla f(\overline{x}^{t-1})\|^2] + \frac{\eta^2 \sigma^2}{N} \nonumber \\
		& + \frac{L^2\eta}{N}\mathbb{E}\Big[\sum_{i=1}^N \|\overline{x}^{t-1} - x_i^{t-1}\|^2\Big] + (1 + L\eta)\eta \mathbb{E}\Big[\|\frac{1}{N}\sum_{i=1}^Ne_i^t\|^2\Big] \nonumber \\
		&- \frac{\eta-2L\eta^2}{N^2} \mathbb{E}\Big[\Big\|\sum_{i=1}^N g_i^{t}\Big\|^2\Big].
	\end{align}
	Let $0 < \eta \leq \frac{1}{2L}$ and use Lemma~\ref{lm:x2avg}, we get:
	\begin{align}
		&\mathbb{E}[f(\overline{x}^t)] \leq \mathbb{E}[f(\overline{x}^{t-1})] -\frac{\eta}{4} \mathbb{E}[\|\nabla f(\overline{x}^{t-1})\|^2] + 4L^2 \eta^3 I^2 \rho^2 \nonumber \\
		&+ \frac{\eta^2 \sigma^2}{N} + (1 + L\eta)\eta \mathbb{E}\Big[\|\frac{1}{N}\sum_{i=1}^Ne_i^t\|^2\Big].
	\end{align}
	Dividing both sides by $\eta$ and rearranging, we obtain:
	\begin{align} \label{eq:convergence_final}
		&\mathbb{E}[\|\nabla f(\overline{x}^{t-1})\|^2] \leq \frac{4}{\eta} (\mathbb{E}[f(\overline{x}^{t-1})] - \mathbb{E}[f(\overline{x}^{t})]) + \frac{4 \eta \sigma^2}{N}\nonumber \\
		& + 16L^2 \eta^2 I^2 \rho^2 + 4(1 + L\eta)\mathbb{E}\Big[\|\frac{1}{N}\sum_{i=1}^N e_i^t\|^2\Big].
	\end{align}
	Summing over $t=1,\dots,T$ and dividing by $T$ gives the desired convergence bound:
	\begin{align}
		&\frac{1}{T} \sum_{t=1}^T \mathbb{E}[\|\nabla f(\overline{x}^{t-1})\|^2] \leq \frac{4}{\eta T} (f(\overline{x}^{0}) - f^*)+ \frac{4 \eta \sigma^2}{N} \nonumber \\
		& + 16L^2 \eta^2 I^2 \rho^2 + 4(1 + L\eta)\mathbb{E}\Big[\|\frac{1}{N}\sum_{i=1}^N e_i^t\|^2\Big].
	\end{align}
	Finally, substituting the upper bound on the averaged reconstruction error from Theorem~\ref{thm:e4error} yields   
	\begin{align}
		&\frac{1}{T} \sum_{t=1}^T \mathbb{E}[\|\nabla f(\overline{x}^{t-1})\|^2] 
		\leq  \frac{4}{\eta T} (f(\overline{x}^{0}) - f^*) + 16L^2 \eta^2 I^2 \rho^2 \nonumber \\
		& + \frac{4 \eta \sigma^2}{N} + \frac{4(1 + L\eta)}{N}\bigg((1 -\delta^2)\rho^2 + (N-1) \tau\bigg).
	\end{align}
	The proof of Theorem~\ref{thm:convergence} is completed.
\end{proof}

}

 \bibliographystyle{IEEEtran}
 \bibliography{IEEEabrv,ref}

\begin{IEEEbiography}[{\includegraphics[width=1in,height=1.25in,clip,keepaspectratio]{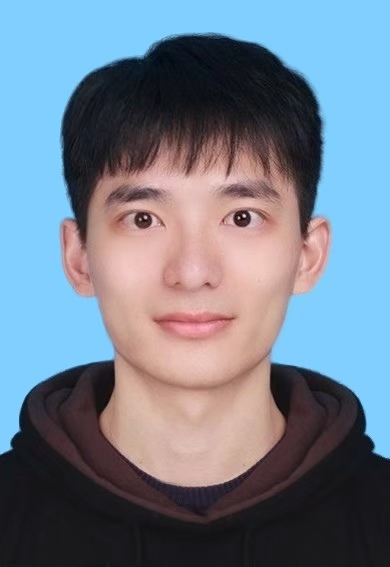}}]{Shenlong Zheng}
received his B.S. in Computer Science and Technology from South China Agricultural University, China, in 2023, and is currently pursuing a Master's Degree in Computer Technology at Jinan University. His research interests include cloud computing, resource management in data centers, and federated learning.
\end{IEEEbiography}
\vspace{-10mm}
\begin{IEEEbiography}[{\includegraphics[width=1in,height=1.25in,clip,keepaspectratio]{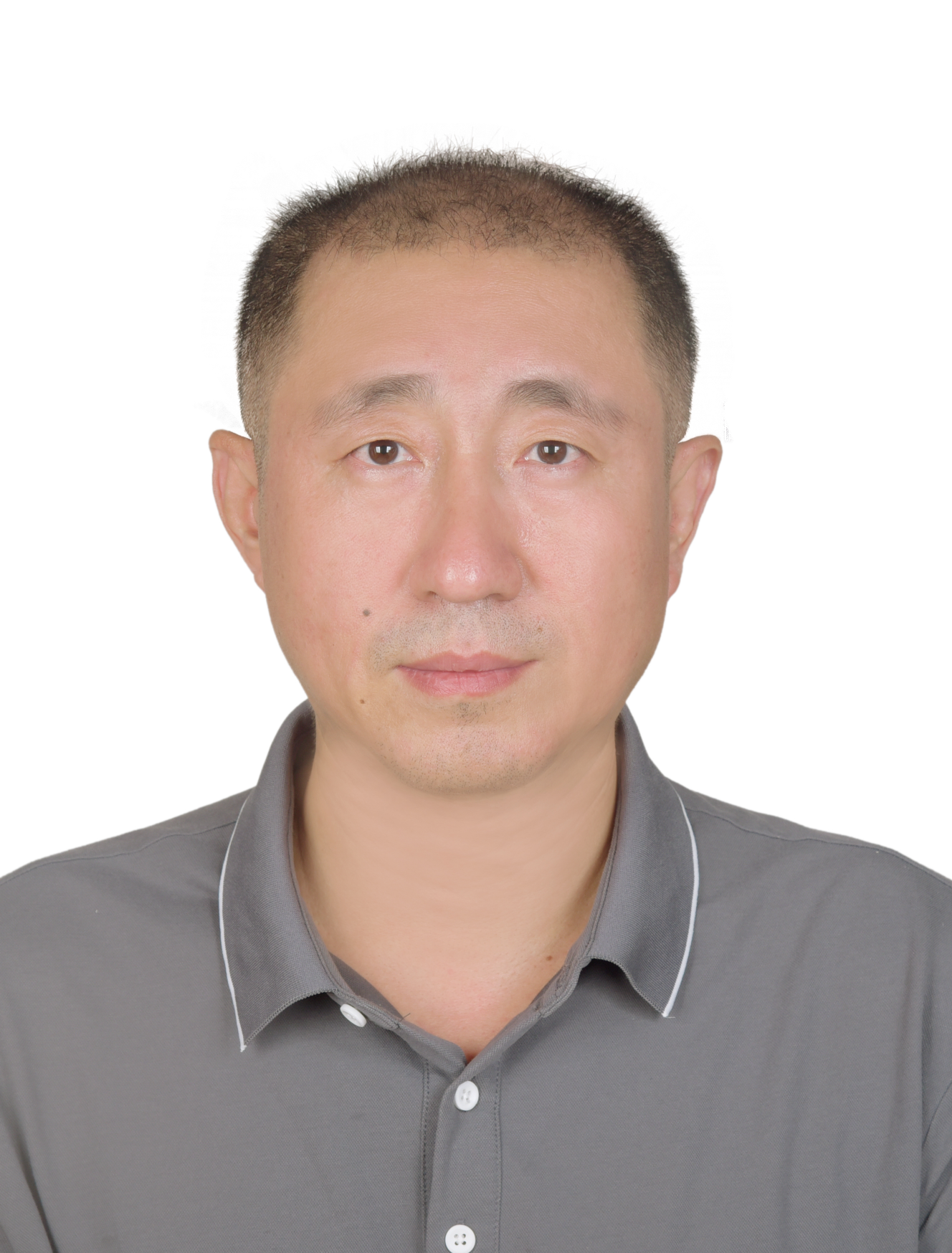}}]{Zhen Zhang}
	received the BS and MS degrees in
	computer science from Jilin University, China, in
	1999 and 2003, and PhD degree in College of Computer Science and Engineering from South China
	University of Technology, China, in 2011. He is
	currently a professor at the Department of Computer
	Science of Jinan University. His research interests
	include graph theory, parallel and distributed processing and complex networks.
\end{IEEEbiography}
\vspace{-10mm}
\begin{IEEEbiography}[{\includegraphics[width=1in,height=1.25in,clip,keepaspectratio]{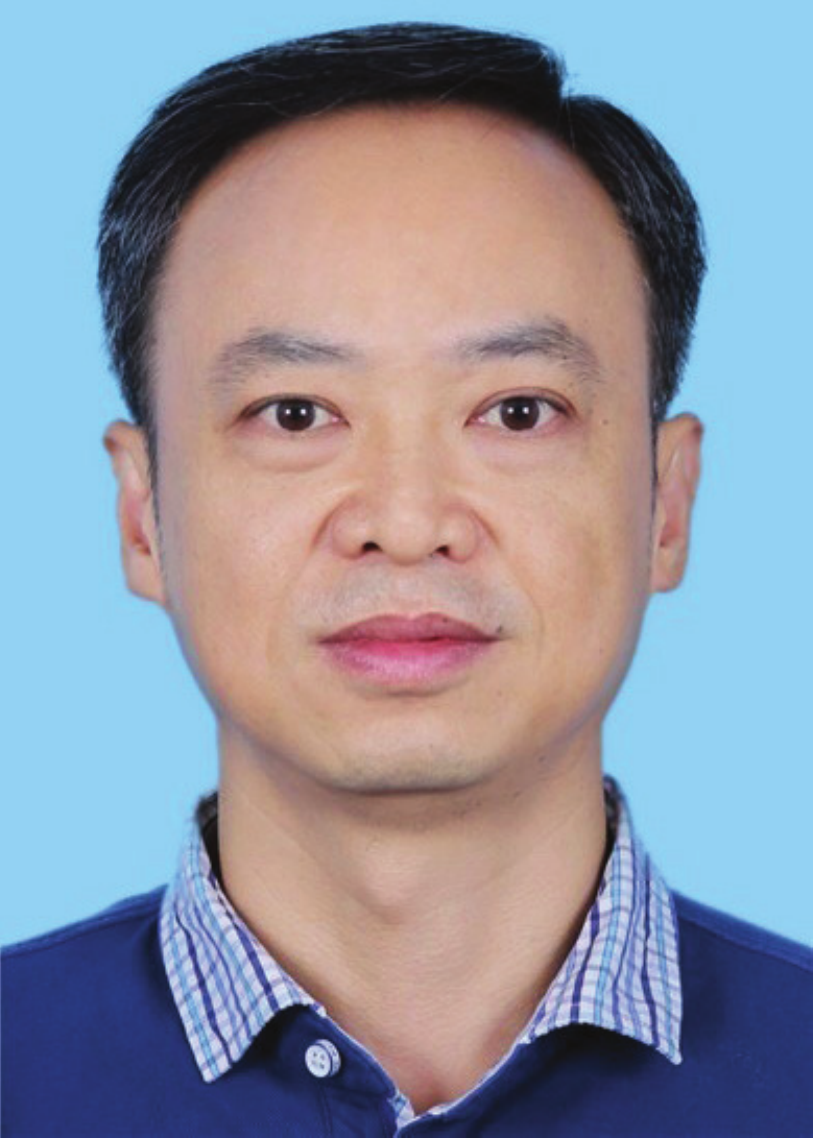}}]{Yuhui Deng}
	is a professor at the computer science
	Department of Jinan University. Before joining Jinan
	University, Dr. Yuhui Deng worked at EMC Corporation as a senior research scientist from 2008 to
	2009. He worked as a research officer at Cranfield
	University in the United Kingdom from 2005 to
	2008. He received his Ph.D. degree in computer
	science from Huazhong University of Science and
	Technology in 2004. His research interests cover information storage, cloud computing, green computing, computer architecture, performance evaluation,
	etc.
\end{IEEEbiography}
\vspace{-10mm}
\begin{IEEEbiography}[{\includegraphics[width=1in,height=1.25in,clip,keepaspectratio]{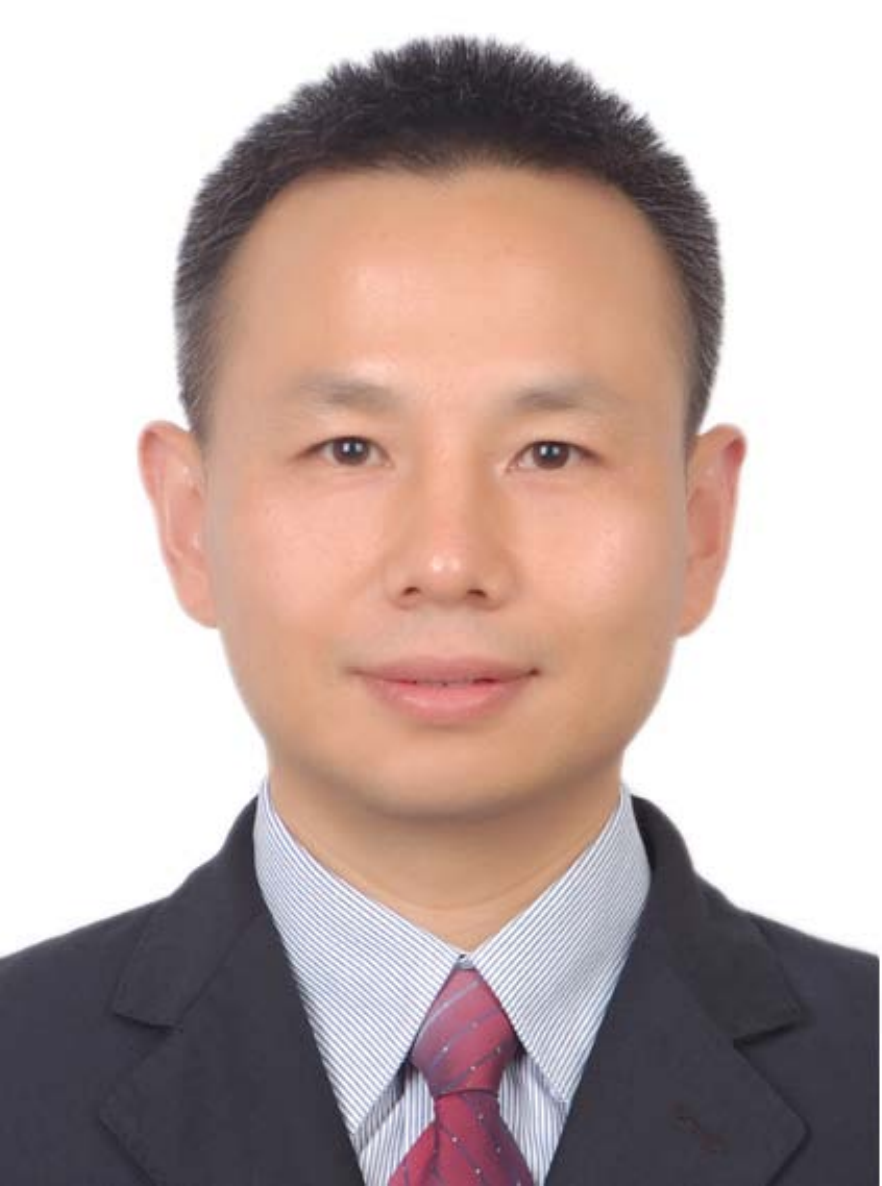}}]{Geyong Min}
	received the BSc degree in computer
	science from the Huazhong University of Science
	and Technology, China, in 1995, and the PhD de
	gree in computing science from the University of
	Glasgow, U.K., in 2003. He is a professor of high
	performance computing and networking with the
	Department of Mathematics and Computer Science
	within the College of Engineering, Mathematics and
	Physical Sciences, University of Exeter, U.K. His
	research interests include future Internet, computer
	networks, wireless communications, multimedia systems, information security, high performance computing, ubiquitous computing, modelling, and performance engineering. He serves as an associate
	editor of the IEEE Transactions on Cloud Computing, IEEE Transactions on
	Sustainable Computing, IEEE Transactions on Computers, and etc.
\end{IEEEbiography}
\vspace{-10mm}
\begin{IEEEbiography}[{\includegraphics[width=1in,height=1.25in,clip,keepaspectratio]{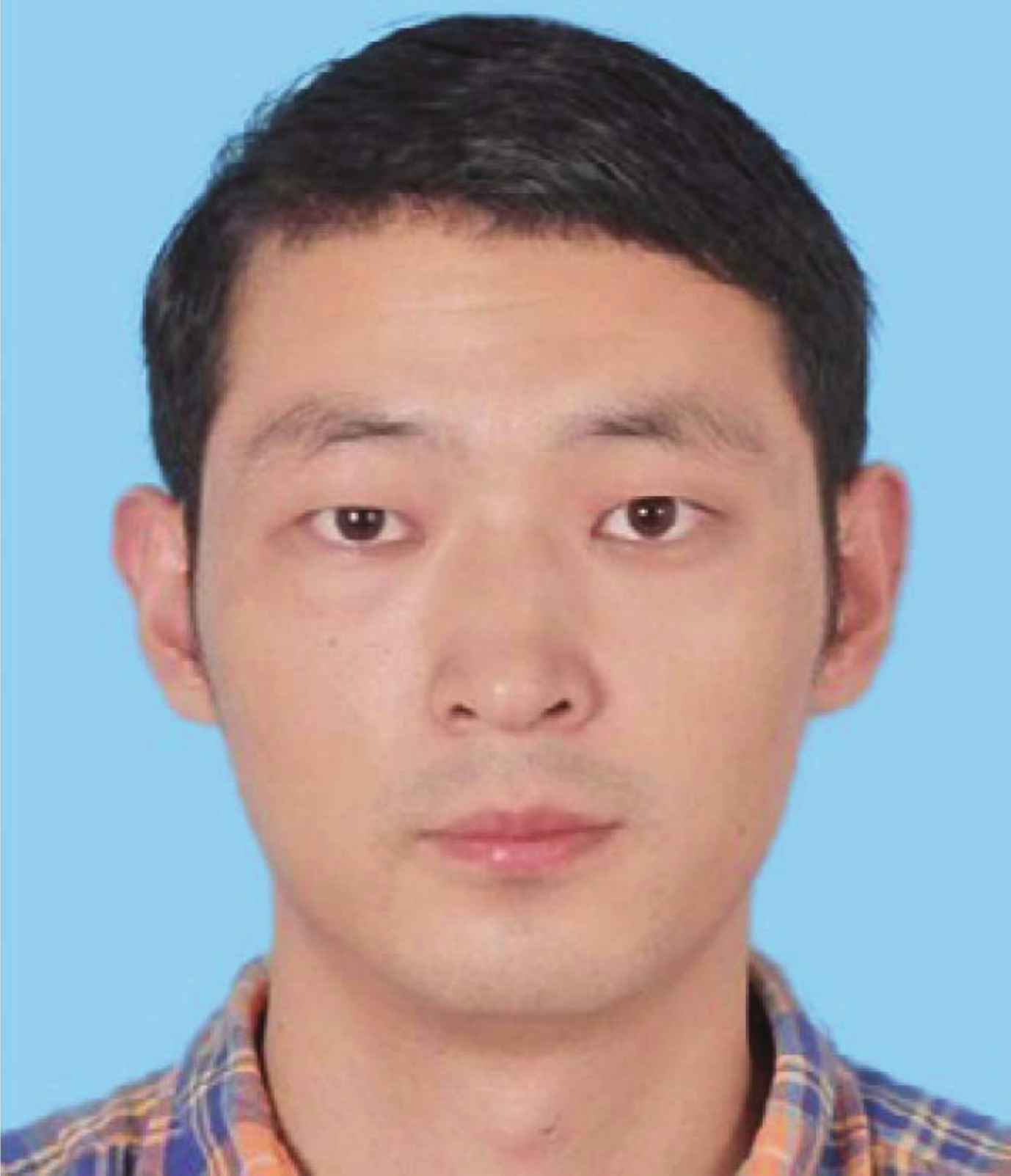}}]{Lin Cui}
	is currently with the Department of
	Computer Science at Jinan University, Guangzhou,
	China. He received the Ph.D. degree from City University of Hong Kong in 2013. He has broad interests
	in networking systems, with focuses on the following topics: cloud data center resource management,
	data center networking, software defined networking
	(SDN), virtualization, programmable data plane and
	so on.
\end{IEEEbiography}
\end{document}